\documentclass[accepted]{uai2023} 

\usepackage[american]{babel}
\usepackage{balance}
\usepackage{natbib} 
\usepackage{mathtools} 
\usepackage{booktabs} 
\usepackage{tikz} 
\usetikzlibrary{automata,positioning,arrows}
\tikzset{node distance=3.5cm, 
every state/.style={ 
semithick,
fill=gray!10},
initial text={}, 
double distance=2pt, 
every edge/.style={ 
                draw,
->,>=stealth, 
auto,
semithick}}

\usepackage{amsthm}
\usepackage{comment}
\usepackage{subcaption}
\usepackage{amsfonts}
\usepackage{algorithm}
\usepackage{algorithmic}
\usepackage{xcolor}
\usepackage{multirow}

\definecolor{gray}{cmyk}{0,0,0,0.8}

\usepackage[capitalise,noabbrev,nameinlink]{cleveref}
\creflabelformat{equation}{#2\textup{#1}#3} 

\title{Reward-Machine-Guided, Self-Paced Reinforcement Learning}

\author[1]{\href{mailto:<cevahir.koprulu@utexas.edu>?Subject=Your UAI 2023 paper}{Cevahir~Koprulu}{}}
\author[1]{Ufuk~Topcu}
\affil[1]{%
    The University of Texas at Austin\\
    USA
}

\usepackage{silence}
\begin{document}
\maketitle

\begin{abstract}
Self-paced reinforcement learning (RL) aims to improve the data efficiency of learning by automatically creating sequences, namely \emph{curricula}, of probability distributions over \emph{contexts}.
However, existing techniques for self-paced RL fail in long-horizon planning tasks that involve temporally extended behaviors. 
We hypothesize that taking advantage of prior knowledge about the underlying task structure can improve the effectiveness of self-paced RL.
We develop a self-paced RL algorithm guided by reward machines, i.e., a type of finite-state machine that encodes the underlying task structure. 
The algorithm integrates reward machines in \begin{enumerate*}[label=\arabic*)]
    \item the update of the policy and value functions obtained by any RL algorithm of choice, and 
    \item the update of the automated curriculum that generates context distributions.
\end{enumerate*}
Our empirical results evidence that the proposed algorithm achieves optimal behavior reliably even in cases in which existing baselines cannot make any meaningful progress. It also decreases the curriculum length and reduces the variance in the curriculum generation process by up to one-fourth and four orders of magnitude, respectively.
\end{abstract}


\newtheorem{definition}{Definition}
\newtheorem{lemma}{Lemma}
\newtheorem{theorem}{Theorem}
\newtheorem{corollary}{Corollary}
\newtheorem{assumption}[theorem]{Assumption}
\newenvironment{proofsketch}{%
\renewcommand{\proofname}{Proof Sketch}\proof}{\endproof}

\newcommand{\Reals}{\mathbb{R}}
\newcommand{\PositiveIntegers}{\mathbb{Z^+}}
\newcommand{\Expectation}{\mathbb{E}}
\newcommand{\KLDivergence}{D_{\text{KL}}}
\newcommand{\Naturals}{\mathbb{N}}
\newcommand{\RLAlgorithm}{\Psi}

\newcommand{\Discount}{\gamma}
\newcommand{\Policy}{\pi}
\newcommand{\CommonAction}{a}
\newcommand{\CommonState}{s}

\newcommand{\PointMassLabelDoorOne}{d1}
\newcommand{\PointMassLabelBox}{b}
\newcommand{\PointMassLabelDoorTwo}{d2}
\newcommand{\PointMassLabelGoal}{g}
\newcommand{\PointMassLabelWall}{w}

\newcommand{\CheetahFlagOne}{f_1}
\newcommand{\CheetahFlagTwo}{f_2}
\newcommand{\CheetahFlagThree}{f_3}

\newcommand{\SwimmerFlagOne}{f_1}
\newcommand{\SwimmerFlagTwo}{f_2}
\newcommand{\SwimmerFlagThree}{f_3}

\newcommand{\Lmdp}{\mathcal{M}}
\newcommand{\LmdpStates}{S}
\newcommand{\LmdpCommonState}{\CommonState}
\newcommand{\LmdpInit}{\phi}
\newcommand{\LmdpActions}{A}
\newcommand{\LmdpCommonAction}{\CommonAction}
\newcommand{\LmdpTransition}{p}
\newcommand{\LmdpRewardFunction}{R}
\newcommand{\LmdpReward}{r}
\newcommand{\LmdpRewardSequence}{\rho}
\newcommand{\LmdpDiscount}{\Discount}
\newcommand{\LmdpLabels}{\mathcal{P}}
\newcommand{\LmdpCommonLabel}{\ell}
\newcommand{\LmdpLabelingFunction}{L}
\newcommand{\LmdpLabelSequence}{\lambda}
\newcommand{\LmdpPolicy}{\Policy}

\newcommand{\RM}{\mathcal{R}}
\newcommand{\RMStates}{Q}
\newcommand{\RMCommonState}{\mathsf{q}}
\newcommand{\RMReward}{r}
\newcommand{\RMInit}{\RMCommonState_{I}}
\newcommand{\RMInputAlphabet}{2^{\LmdpLabels}}
\newcommand{\RMOutputAlphabet}{O}
\newcommand{\RMTransitionFunction}{\delta_{\RMCommonState}}
\newcommand{\RMOutputFunction}{\delta_{\LmdpReward}}
\newcommand{\RMCommonLabel}{\LmdpCommonLabel}
\newcommand{\LogicFormula}{\rho}

\newcommand{\context}{c}
\newcommand{\Cmdp}{\mathcal{\Bar{M}}}
\newcommand{\CmdpStates}{S}
\newcommand{\CmdpActions}{A}
\newcommand{\CmdpContextSpace}{\mathcal{C}}
\newcommand{\CmdpMapping}{\mathsf{M}}
\newcommand{\CmdpCommonState}{\CommonState}
\newcommand{\CmdpCommonAction}{\CommonAction}
\newcommand{\CmdpTransition}{p_{\context}}
\newcommand{\CmdpRewardFunction}{R_{\context}}
\newcommand{\CmdpInitialDistribution}{\phi_{\context}}
\newcommand{\CmdpPolicy}{\Policy}
\newcommand{\CmdpDiscount}{\Discount}
\newcommand{\CmdpDimensions}{D}
\newcommand{\CmdpReward}{r}

\newcommand{\PolicyParameter}{\omega}
\newcommand{\ValueFunction}{V_{\PolicyParameter}}
\newcommand{\TargetContextDistribution}{\varphi}

\newcommand{\KLCoefficient}{\alpha}
\newcommand{\ContextDistribution}{\varrho}
\newcommand{\ContextDistributionParameter}{\nu}
\newcommand{\RelativeEntropyBound}{\epsilon}
\newcommand{\SPRLTrajectorySet}{\mathcal{D}}
\newcommand{\EstimatedValueFunction}{\hat{V}_{\PolicyParameter}}
\newcommand{\KLPenaltyProportion}{\zeta}
\newcommand{\KLPenaltyOffset}{K_{\alpha}}
\newcommand{\NumberOfIterations}{K}
\newcommand{\NumberOfRollouts}{N}
\newcommand{\Trajectory}{\tau}
\newcommand{\ContextUpdateOffset}{K_{\text{OFFSET}}}
\newcommand{\NumberOfStepsBetweenUpdates}{n_{\text{STEP}}}
\newcommand{\STDLowerBound}{\sigma_{\text{LB}}}
\newcommand{\KLLowerBound}{D_{\text{KL}_{LB}}}
\newcommand{\SPRLInitMean}{\mu_{\text{INIT}}}
\newcommand{\SPRLInitVar}{\Sigma_{\text{INIT}}}
\newcommand{\SPRLTargetMean}{\mu_{\text{TARGET}}}
\newcommand{\SPRLTargetVar}{\Sigma_{\text{TARGET}}}

\newcommand{\LCmdp}{\Cmdp^{\LmdpLabelingFunction}}
\newcommand{\LCmdpStates}{\CmdpStates}
\newcommand{\LCmdpActions}{\CmdpActions}
\newcommand{\LCmdpContextSpace}{\CmdpContextSpace}
\newcommand{\LCmdpMapping}{\CmdpMapping^{\LmdpLabelingFunction}}
\newcommand{\LCmdpCommonState}{\CommonState}
\newcommand{\LCmdpCommonAction}{\CommonAction}
\newcommand{\LCmdpTransition}{\CmdpTransition}
\newcommand{\LCmdpRewardFunction}{\CmdpRewardFunction^{\LmdpLabelingFunction}}
\newcommand{\LCmdpInitialDistribution}{\CmdpInitialDistribution}
\newcommand{\LCmdpPolicy}{\Policy}
\newcommand{\LCmdpDiscount}{\Discount}
\newcommand{\LCmdpLabels}{\LmdpLabels}
\newcommand{\LCmdpLabelingFunction}{\LmdpLabelingFunction_{\context}}
\newcommand{\LCmdpDimensions}{\CmdpDimensions}
\newcommand{\LCmdpCommonLabel}{\LmdpCommonLabel}
\newcommand{\LCmdpReward}{\LmdpReward}
\newcommand{\LCmdpHistory}{h}
\newcommand{\LCmdpNumberOfDimensions}{\Gamma}

\newcommand{\Pmdp}{\Cmdp^{\LmdpLabelingFunction}_{\RM}}
\newcommand{\PmdpStates}{\bar{\CmdpStates}}
\newcommand{\PmdpActions}{\CmdpActions}
\newcommand{\PmdpContextSpace}{\CmdpContextSpace}
\newcommand{\PmdpMapping}{\bar{\CmdpMapping}^{\LmdpLabelingFunction}}
\newcommand{\PmdpCommonState}{\bar{\CommonState}}
\newcommand{\PmdpCommonAction}{\CommonAction}
\newcommand{\PmdpTransition}{\bar{p}_{\context}}
\newcommand{\PmdpRewardFunction}[1][\context]{\bar{R}^{\LmdpLabelingFunction}_{#1}}
\newcommand{\PmdpInitialDistribution}{\bar{\phi}_{\context}}
\newcommand{\PmdpPolicy}{\Policy}
\newcommand{\PmdpDiscount}{\Discount}
\newcommand{\PmdpLabels}{\LmdpLabels}
\newcommand{\PmdpLabelingFunction}{\LCmdpLabelingFunction}
\newcommand{\PmdpDimensions}{\LCmdpDimensions}
\newcommand{\PmdpCommonLabel}{\LmdpCommonLabel}
\newcommand{\PmdpReward}{\bar{\LCmdpReward}}
\newcommand{\PmdpTrajectory}{\bar{\Trajectory}}
\newcommand{\RMContextMapping}{\mathsf{F}}
\newcommand{\RMContextMappingCommonOutput}{\mathsf{f}}

\newcommand{\RMMDPContextSet}{\mathcal{G}}
\newcommand{\RMMDPContextMapping}{\mathsf{H}_{min}}
\newcommand{\RMMDPContextSetALL}{\Gamma}

\newcommand{\GoalGANdNoise}{\delta_{\text{NOISE}}}
\newcommand{\GoalGANnRollout}{n_{\text{ROLLOUT}}^{\text{GG}}}
\newcommand{\GoalGANpSuccess}{p_{\text{SUCCESS}}}

\newcommand{\ALPGMMpRandom}{p_{\text{RAND}}}
\newcommand{\ALPGMMnRollout}{n_{\text{ROLLOUT}}^{\text{AG}}}
\newcommand{\ALPGMMsBuffer}{s_{\text{BUFFER}}}


\section{Introduction}

\begin{figure}[t]
\centering
\includegraphics[width=0.48\textwidth]{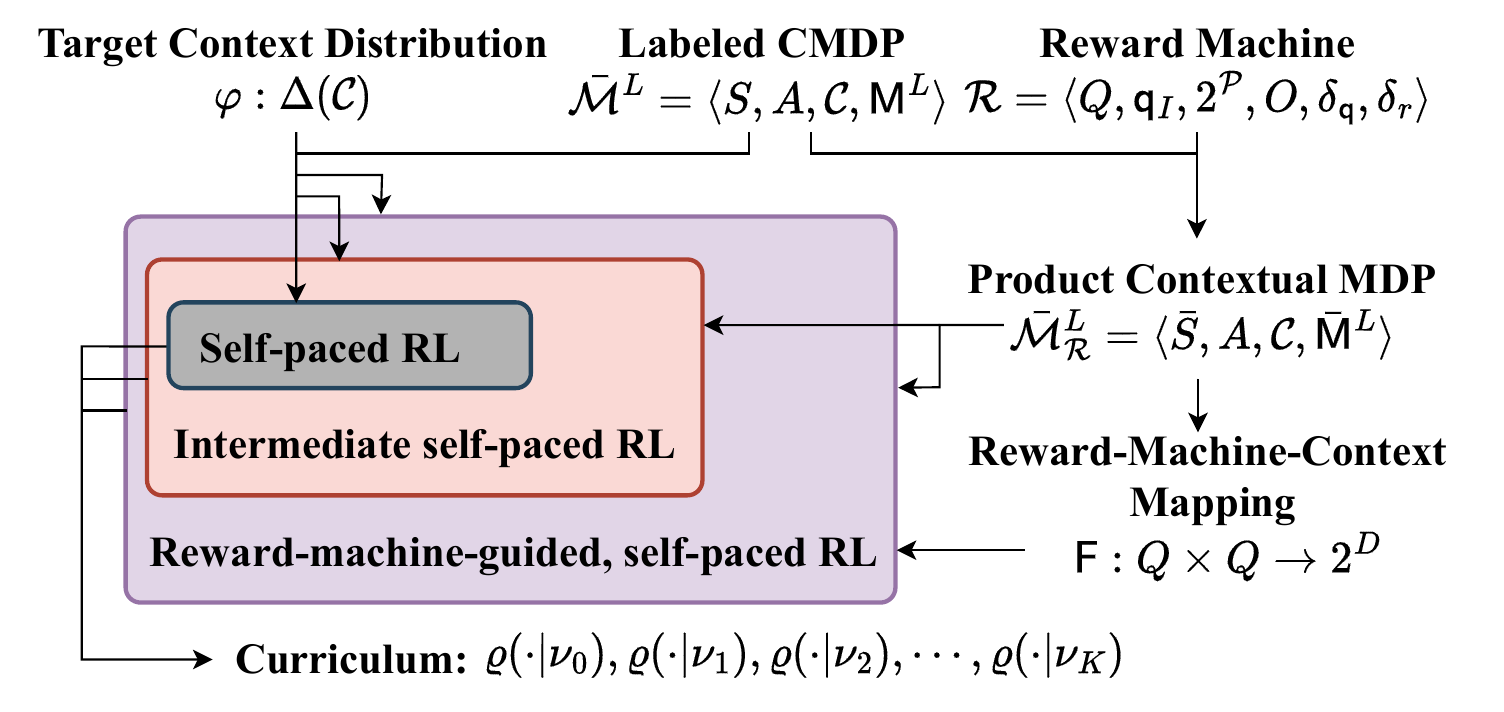}
\caption{Workflow diagram for an existing self-paced RL approach, and two methods that we propose: intermediate self-paced RL and reward-machine-guided, self-paced RL.
}
\label{fig:workflow}
\end{figure}
The design of task sequences, i.e., curricula \citep{bengio2009curriculum}, aims to reduce the sample complexity of teaching reinforcement learning (RL) agents complex behaviors. 
Given a \emph{target} task, a common curriculum design approach is to begin with \emph{easier} tasks and increase the difficulty in a gradual manner, which requires domain expertise to define what is easy or hard \citep{narvekar2020curriculum}. To eliminate the need for manual curriculum design, many studies such as \citet{baranes2010intrinsically,svetlik2017automatic, NIPS2017_453fadbd} focus on automating the process of curriculum generation. \citet{klink2020self_SPCRL} adopt self-paced learning \citep{NIPS2010_e57c6b95} in RL by developing an algorithm that creates a sequence of probability distributions over \emph{contexts} \citep{hallak2015contextual}. The dynamics, the reward function, and the initial state distribution of an environment change with respect to the context. Given a \emph{target context distribution},
a self-paced RL algorithm iteratively generates \emph{context distributions} that maximizes the expected discounted return, regularized by the Kullback-Leibler (KL) divergence from the target context distribution.

Although empirical evidence by \citet{klink2021probabilistic} suggests that self-paced RL outperforms the state-of-the-art curriculum learning methods \citep{florensa2018automatic,portelas2020teacher}, 
existing self-paced RL approaches work poorly in long-horizon planning tasks, which involve temporally extended behaviors. 
We focus on tasks where the reward depends on the history of states and actions. In other words, the reward function of such a long-horizon planning task is non-Markovian. 
A remedy to tasks that require temporally extended behaviors is to expose the high-level structural relationships to the agent \citep{singh1992reinforcement, parr1997reinforcement}.
\citet{icarte2018using} use a type of finite-state machine, called \emph{reward machines}, as the high-level structural knowledge to encode non-Markovian reward functions in RL.  

We claim that exploiting the high-level structural knowledge about a long-horizon planning task can improve self-paced RL. To this end, we study self-paced RL for long-horizon planning tasks in which such knowledge is available a priori to the agent in the form of reward machines. Specifically, we focus on contextual long-horizon planning tasks, where the context parameterizes the dynamics and the non-Markovian reward function. The underlying temporal task structure remains the same irrespective of the context, hence a reward machine can encode all possible non-Markovian reward functions. We define a labeled contextual Markov decision process (MDP) to model such long-horizon planning tasks (see \cref{fig:workflow}).

\paragraph{Contribution.} Our contribution is three-fold. \begin{enumerate*}[label=\arabic*)]
    \item We propose an intermediate self-paced RL algorithm that combines a labeled contextual MDP and its reward machine in a product contextual MDP to update the policy and value functions of an RL agent. 
    \item We establish a mapping that, given a transition in the reward machine, outputs the smallest set of context parameters, that determine whether the transition, namely a high-level event, occurs or not. 
    \item We develop a reward-machine-guided, self-paced RL algorithm that exploits reward machines not only to update the policy and value functions but also to navigate the generation of curricula via the proposed mapping (see \cref{fig:workflow}). 
\end{enumerate*}

Our experiments conclude that, first, proposed reward-machine-guided and intermediate self-paced RL algorithms enable RL agents to accomplish long-horizon planning tasks by encoding non-Markovian reward functions as reward machines, whereas state-of-the-art automated curriculum generation methods fail to do so; and, second, guiding curriculum generation via a \emph{reward-machine-context mapping} not only boosts learning speed reliably but also stabilizes the curriculum generation process by reducing curricula variance by up to four orders of magnitude, and thus avoid inefficient exploration of the curriculum space.

\section{Related Work}

We propose an automated curriculum generation method, that exploits high-level structural knowledge about long-horizon planning tasks. Our work falls under two subjects.

\paragraph{Curriculum learning for RL.} Automatically generating curricula in RL modifies the configuration of the environment iteratively to accelerate convergence to optimal policies. As we do, many studies in the literature consider a curriculum as a sequence of distributions over environment configurations.
\citet{pmlr-v78-florensa17a} proposes the generation of distributions over initial states that iteratively get further away from goal states. 
Others focus on goal-conditioned RL where a curriculum is a sequence of distributions over goal states that optimize value disagreement \citep{NEURIPS2020_566f0ea4}, feasibility and coverage of goal states \citep{racaniere2019automated}, intrinsic motivation \citep{baranes2010intrinsically,portelas2020teacher}, and intermediate goal difficulty \citep{florensa2018automatic}. 
For procedural content generation environments, curricula prioritize levels with higher learning potential \citep{jiang2021prioritized,jiang2021replay}.
In comparison, self-paced RL is adopted from supervised learning where training samples are automatically ordered in increasing complexity \citep{NIPS2010_e57c6b95,jiang2015self}. 
\citet{ren2018self} considers curricula as a sequence of environment interactions and proposes a self-paced mechanism that minimizes coverage penalty.
\citet{eimer2021self}'s work generates a sequence of contexts, not distributions, with respect to their capacity of value improvement. \citet{klink2020self_SPCRL,NEURIPS2020_68a97503,klink2021probabilistic,klink2022curriculum} formulate the generation of curricula as interpolations between distributions over contexts. Similarly, \citet{chen2021variational} study interpolations between task distributions, but not under the self-paced RL framework.

\paragraph{Incorporating high-level structural knowledge into RL.} \citet{singh1992reinforcement,parr1997reinforcement,sutton1999between,dietterich2000hierarchical} propose the idea of incorporating high-level structural knowledge to decompose a task into a hierarchy of subtasks. The proposed hierarchy allows the agent to learn a meta-controller that chooses between subtasks to pursue, and a low-level controller that acts in the chosen subtask. Another way to incorporate such knowledge is to capture temporal abstractions in long-horizon planning tasks via temporal logic \citep{bacchus1996rewarding,li2017reinforcement,littman2017environment}, or reward machines  \citep{icarte2018using,camacho2019ltl}, which address MDPs with non-Markovian structures. We investigate a multi-task setting with non-Markovian reward functions and propose an automated curriculum generation approach that uses reward machines, 1) to encode non-Markovian reward functions; and 2) to guide the curriculum generation process. Similar to curriculum learning, \citet{toro2018teaching,xu2019transfer,kuo2020encoding,zheng2022lifelong,velasquez2021dynamic}
study the use of temporal logic and reward machines in topics such as generalization, transfer learning, and multi-task learning.  

\section{Preliminaries}

\begin{figure}[t]
\centering
\includegraphics[width=0.35\textwidth]{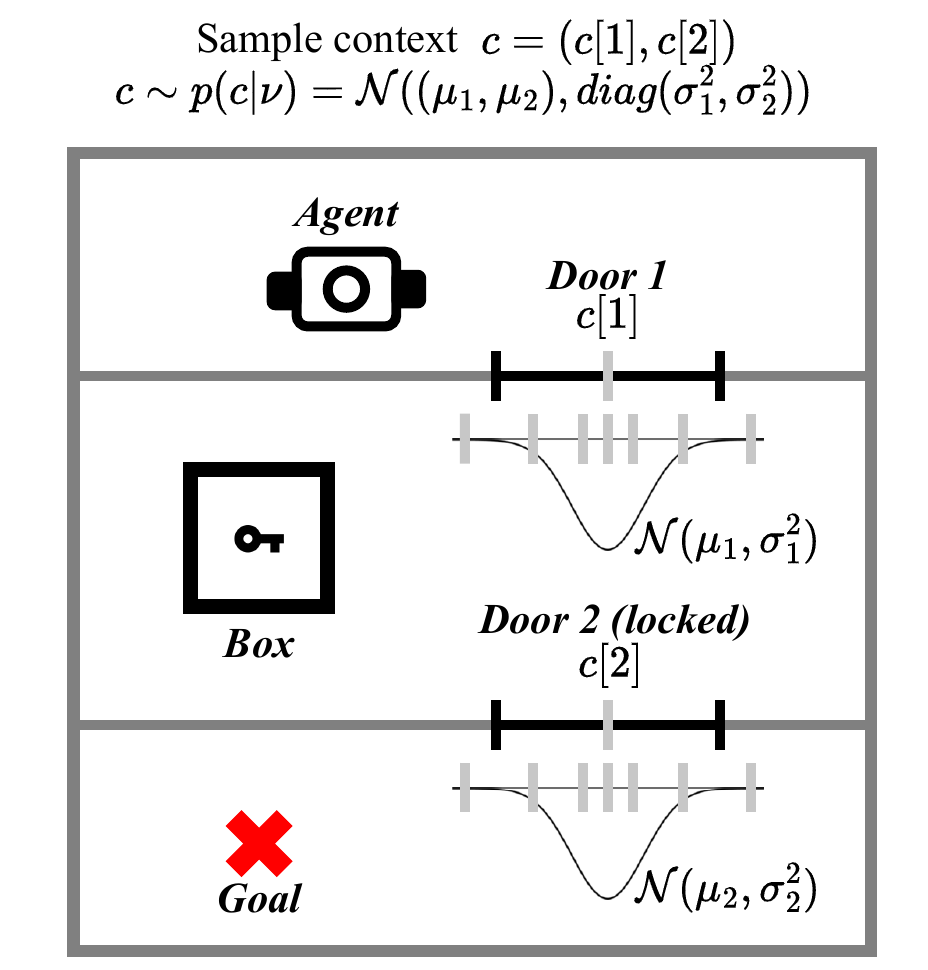}
\caption{The two-door environment: A context $\context=(c[1],c[2])$ determines the positions of \emph{Door 1} and \emph{Door 2} with $\context[1]$ and $\context[2]$, respectively.}
\label{fig:twodoor_env}
\end{figure}
In this section, we provide the background for our problem of interest.
We illustrate a two-door environment, (see \cref{fig:twodoor_env}), which we will revisit throughout the paper.
The agent has to complete 4 subtasks in the following order: (1) Passing through \emph{Door 1}, (2) getting a key from \emph{Box}, (3) opening \emph{Door 2} with the key, and (4) arriving at \emph{Goal}. The agent has to avoid hitting the walls that separate the rooms.

\subsection{Labeled MDPs and Reward Machines}
\begin{definition}
\label{def:LMDP}
A labeled Markov decision process \cite{xu2020joint} is a tuple
$\Lmdp = \langle \LmdpStates, \LmdpActions, \LmdpTransition, \LmdpRewardFunction, \LmdpInit, \LmdpDiscount, \LmdpLabels, \LmdpLabelingFunction\rangle$
consisting of a state space  $\LmdpStates$, 
an action space $\LmdpActions$,
a probabilistic transition function
$\LmdpTransition \colon \LmdpStates\times \LmdpActions \times \LmdpStates \rightarrow [0,1]$, and
an initial state distribution $\LmdpInit: \LmdpStates \to [0,1]$. 
A reward function $\LmdpRewardFunction: (\LmdpStates \times \LmdpActions)^+ \times \LmdpStates \rightarrow \Reals$,
and a discount factor $\LmdpDiscount \in [0,1)$ specify the returns to the agent.
A finite set $\LmdpLabels$ of propositional variables, 
and a labeling function $\LmdpLabelingFunction: \LmdpStates\times \LmdpActions\times \LmdpStates\rightarrow \RMInputAlphabet$
determine the set of high-level events that the agent sees in the environment.
\end{definition}

A \emph{policy} $\LmdpPolicy$ is a function mapping states in $\LmdpStates$ to a probability distribution over actions in $\LmdpActions$.
At state $s\in \LmdpStates$, an agent using policy $\LmdpPolicy$ picks an action $\LmdpCommonAction$ with probability $\LmdpPolicy(\LmdpCommonState, \LmdpCommonAction)$,
and the new state $\LmdpCommonState'$ is chosen with probability $\LmdpTransition(\LmdpCommonState, \LmdpCommonAction, \LmdpCommonState')$. 

For a fixed context, we can model the two-door environment as a labeled MDP. The states are the coordinates of the agent and the actions are moving in the four cardinal directions, whereas the transitions are deterministic. The agent receives the labels $\{\PointMassLabelDoorOne\}$, $\{\PointMassLabelBox\}$, $\{\PointMassLabelDoorTwo\}$, $\{\PointMassLabelGoal\}$, and $\{\PointMassLabelWall\}$ when it moves onto \emph{Door 1}, \emph{Box}, \emph{Door 2}, \emph{Goal}, and the walls, respectively.

\begin{definition}
A reward machine \cite{icarte2018using} $\RM = \langle \RMStates, \RMInit, \RMInputAlphabet, \RMOutputAlphabet, \RMTransitionFunction, \RMOutputFunction \rangle$ consists of a finite, nonempty set $\RMStates$ of states, an initial state $\RMInit \in \RMStates$, an input alphabet $\RMInputAlphabet$, an output alphabet $\RMOutputAlphabet \subset \Reals$, a deterministic transition function $\RMTransitionFunction \colon \RMStates \times \RMInputAlphabet \to \RMStates$, and an output function $\RMOutputFunction \colon \RMStates \times \RMInputAlphabet \to \RMOutputAlphabet$. 
\end{definition}

Reward machines encode non-Markovian reward functions. The run $\RMCommonState_0 (\LmdpCommonLabel_1, \LmdpReward_1) \RMCommonState_1 (\LmdpCommonLabel_2, \LmdpReward_2)\ldots (\LmdpCommonLabel_k, \LmdpReward_k) \RMCommonState_{k+1}$ of a reward machine $\RM$ on a label sequence $\LmdpCommonLabel_1\ldots \LmdpCommonLabel_k\in (\RMInputAlphabet)^*$ is a sequence of states and label-reward pairs such that $\RMCommonState_0 = \RMInit$,  $\RMTransitionFunction(\RMCommonState_i, \LmdpCommonLabel_i) = \RMCommonState_{i+1}$ and $\RMOutputFunction(\RMCommonState_i,\LmdpCommonLabel_i) = \LmdpReward_i$ for all $i\in\{0,\ldots, k\}$.
The reward machine $\RM$ produces a sequence of rewards from an input label sequence as $\RM(\LmdpCommonLabel_1\ldots\LmdpCommonLabel_k) = \LmdpReward_1 \ldots \LmdpReward_k$. 
We say that a reward machine $\RM$ \emph{implements} the reward function $R$ of an MDP if for every trajectory $s_0 a_0 \ldots s_k a_k s_{k+1}$ and the corresponding label sequence $\LmdpCommonLabel_1\ldots \LmdpCommonLabel_k$, the reward sequence equals $\RM(\LmdpCommonLabel_1\ldots \LmdpCommonLabel_k)$.
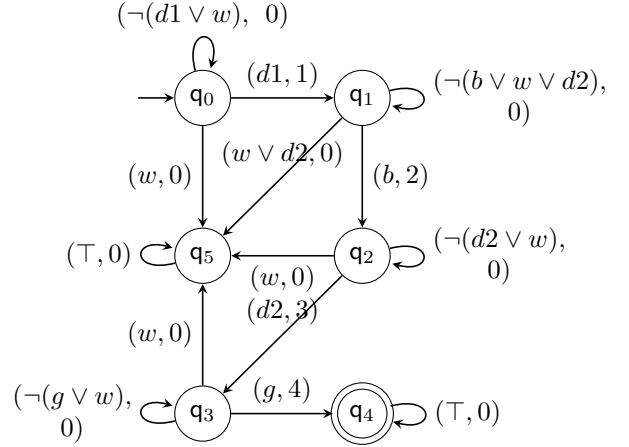
\begin{figure}
    \centering
    \begin{tikzpicture}[
    shorten >=.1pt,node distance=2.1
    cm,on grid,auto, align=center,
    state/.style={circle, draw, minimum size=.1cm}
] 
        \node[state, initial] (q0) {$\RMCommonState_0$};
        \node[state, right of=q0] (q1) {$\RMCommonState_1$};
        \node[state, below of=q0] (q5) {$\RMCommonState_5$};
        \node[state, right of=q5] (q2) {$\RMCommonState_2$};
        \node[state, below of=q5] (q3) {$\RMCommonState_3$};
        \node[state, accepting, right of=q3] (q4) {$\RMCommonState_4$};
        \draw (q0) edge node {\tt $(\PointMassLabelDoorOne, 1)$}(q1);
        \draw (q1) edge node {\tt $(\PointMassLabelBox, 2)$}(q2);
        \draw (q2) edge[above] node {\tt $(\PointMassLabelDoorTwo, 3)$}(q3);
        \draw (q3) edge node {\tt $(\PointMassLabelGoal, 4)$}(q4);
        
        \draw (q0) edge[left] node {\tt $(\PointMassLabelWall, 0)$}(q5);
        \draw (q1) edge[above] node {\tt $(\PointMassLabelWall \vee \PointMassLabelDoorTwo, 0)$}(q5);
        \draw (q2) edge node {\tt $(\PointMassLabelWall, 0)$}(q5);
        \draw (q3) edge node {\tt $(\PointMassLabelWall, 0)$}(q5);

        \draw (q4) edge[loop right] node {\tt $(\top, 0)$} (q4);
        \draw (q5) edge[loop left] node {\tt $(\top, 0)$} (q5);
        
        \draw (q0) edge[loop above] node {\tt$(\neg (\PointMassLabelDoorOne \vee \PointMassLabelWall),$ $0)$} (q0);
        \draw (q1) edge[loop right] node {\tt$(\neg (\PointMassLabelBox \vee \PointMassLabelWall \vee \PointMassLabelDoorTwo),$\\ $0)$} (q1);
        \draw (q2) edge[loop right] node {\tt$(\neg (\PointMassLabelDoorTwo \vee \PointMassLabelWall),$ \\ $0)$} (q2);
        \draw (q3) edge[loop left] node {\tt$(\neg (\PointMassLabelGoal \vee \PointMassLabelWall),$ \\ $0)$} (q3);
    \end{tikzpicture}
        \caption{Reward machine of the two-door environment}
        \label{fig:RM_two_door}
\end{figure}

The reward machine of the two-door environment is in \cref{fig:RM_two_door}. Every node is a state in the reward machine. An edge $(\RMCommonState_i,\RMCommonState_j)$ with tuple $(\LogicFormula,\LmdpReward)$ indicates that given a label $\LmdpCommonLabel\in\RMInputAlphabet$ satisfying propositional formula $\LogicFormula$, the transition from state $\RMCommonState_i$ to $\RMCommonState_j$ yields reward $\CmdpReward=\RMOutputFunction(\RMCommonState_i,\LmdpCommonLabel)$. For instance, the transition ($\RMCommonState_1,\RMCommonState_5$) with ($\PointMassLabelWall \lor \PointMassLabelDoorTwo, 0$) in \cref{fig:RM_two_door} occurs if a label $\LmdpCommonLabel$ satisfies $\LogicFormula=\PointMassLabelWall \lor \PointMassLabelDoorTwo$, i.e., if the agent moves into a wall $\PointMassLabelWall$ or passes the second door $\PointMassLabelDoorTwo$, yielding a reward of $0$. The agent receives rewards of 1, 2, 3, and 4 upon completing the four subtasks, respectively. We note that there can be multiple reward machines that encode an MDP's reward function, and reward machines may differ on a label sequence that does not correspond to any trajectory of the MDP.

\subsection{Contextual MDPs}
\begin{definition}
A contextual Markov decision process (CMDP) \cite{hallak2015contextual} $\Cmdp=\langle\CmdpStates,\CmdpActions,\CmdpContextSpace, \CmdpMapping \rangle$ is defined by a state space $\CmdpStates$, an action space $\CmdpActions$, a context space $\CmdpContextSpace$ and a mapping $\CmdpMapping$ from $\CmdpContextSpace$ to MDP parameters. $\Cmdp$ represents a family of MDPs parameterized by contexts $c \in\CmdpContextSpace\subseteq\Reals^n$, $n\in\PositiveIntegers$. An MDP in this family is a tuple $\CmdpMapping(c)=\langle\CmdpStates,\CmdpActions, \CmdpTransition, \CmdpRewardFunction, \CmdpInitialDistribution, \CmdpDiscount\rangle$ that shares the same state space $\CmdpStates$ and action space $\CmdpActions$ with other members, but its probabilistic transition function $p_c:\CmdpStates\times\CmdpActions\times\CmdpStates\to[0,1]$, reward function $\CmdpRewardFunction:\CmdpStates\times\CmdpActions\to[0,1]$ and initial state distribution $\CmdpInitialDistribution:\CmdpStates\to[0,1]$ depend on $c$.
\end{definition}

CMDPs appear in the multi-task RL literature to model tasks, where transition functions, reward functions, and initial state distributions are parameterized via contexts. Although the two-door environment in \cref{fig:twodoor_env} has a context that determines the door positions, which in return affects the transition and reward functions, the CMDP framework fails to model such an environment as the two-door environment has a non-Markovian reward function. We present a new MDP formulation to address this limitation in \cref{section:problem}.

Given a CMDP $\Cmdp$, \emph{contextual RL} \cite{hallak2015contextual} aims to learn a policy that maximizes the expectation of the value of an initial state in context $\context$ sampled from a target context distribution $\TargetContextDistribution:\CmdpContextSpace\to[0,1]$, namely, $\max_{\omega}\Expectation_{\TargetContextDistribution(\context),\CmdpInitialDistribution(\CmdpCommonState)}[\ValueFunction(\CmdpCommonState,\context)],$
where $\ValueFunction(\CmdpCommonState,\context)$ is the value of state $\CmdpCommonState$ in context $\context$ and encodes the expected discounted return obtained by following the policy $\CmdpPolicy_{\PolicyParameter}(\CmdpCommonAction|\CmdpCommonState,\context)$ as $\ValueFunction(\CmdpCommonState,\context)=\Expectation_{\CmdpPolicy_{\PolicyParameter}(\CmdpCommonAction|\CmdpCommonState,\context)}[\CmdpRewardFunction(\CmdpCommonState,\CmdpCommonAction)+\CmdpDiscount \Expectation_{\CmdpTransition(\CmdpCommonState'|\CmdpCommonState,\context)}[\ValueFunction(\CmdpCommonState',\context)]].$

\subsection{Self-Paced Reinforcement Learning}

\cite{klink2020self_SPCRL, NEURIPS2020_68a97503, klink2021probabilistic} propose algorithms under the \emph{Self-paced RL} framework to address the contextual RL problem. \emph{Self-paced RL} aims to iteratively generate a sequence of context distributions by maximizing the expected performance with respect to the current distribution, which is regularized by the KL divergence from the target context distribution, namely,
\begin{align}
    \max_{\ContextDistributionParameter, \PolicyParameter} \quad & \Expectation_{\ContextDistribution(\context|\ContextDistributionParameter), \CmdpInitialDistribution(\CmdpCommonState)}[\ValueFunction(\CmdpCommonState,\context)] - \KLCoefficient \KLDivergence(\ContextDistribution(\context|\ContextDistributionParameter)\:||\:\TargetContextDistribution(\context)) \nonumber \\
    \textrm{s.t.} \quad &  \KLDivergence(\ContextDistribution(\context|\ContextDistributionParameter)\:||\:\ContextDistribution(\context|\ContextDistributionParameter_{prev})) \leq \RelativeEntropyBound,
    \label{eq:SPRL_objective}
\end{align}
where $\ContextDistribution(c|\ContextDistributionParameter)$ and $\KLCoefficient$ are the current context distribution parameterized by $\ContextDistributionParameter$ and the regularization coefficient, respectively. \citet{klink2020self_SPCRL} introduce the constraint in (\ref{eq:SPRL_objective}) to restrict the divergence of the current context distribution from the previous context distribution parameterized by $\ContextDistributionParameter_{prev}$. \citet{NEURIPS2020_68a97503} propose a way to estimate the expectation in (\ref{eq:SPRL_objective}). By following policy $\CmdpPolicy_{\PolicyParameter}$, they collect a set $\SPRLTrajectorySet=\{(\context_i, \Trajectory_i) | \context_i \sim \ContextDistribution(\context|\ContextDistributionParameter_{prev}),i\in\{1,2,\cdots,M\}\}$ of trajectories $\Trajectory_i=(\CmdpCommonState_{i,0},\CmdpCommonAction_{i,0},\CmdpReward_{i,1},\CmdpCommonState_{i,1}),\cdots,(\CmdpCommonState_{i,T_i-1},\CmdpCommonAction_{i,T_i-1},\CmdpReward_{i,T_i},\CmdpCommonState_{i,T_i})$, where $\CmdpReward_{i,t+1}=R_{\context_{i}}(\CmdpCommonState_{i,t},\CmdpCommonAction_{i,t},\CmdpCommonState_{i,t+1})$ is the reward received at time $t+1$ in trajectory $\Trajectory_i$.
Then, they use the cumulative sum of discounted rewards collected in trajectories to obtain an unbiased estimator of the expectation as
\begin{equation}
    \frac{1}{\NumberOfRollouts} \sum_{i=1}^{\NumberOfRollouts} \frac{\ContextDistribution(\context_i|\ContextDistributionParameter)}{\ContextDistribution(\context_i|\ContextDistributionParameter_{prev})} \sum_{t=0}^{T_i-1} \CmdpDiscount^t \CmdpReward_{i, t+1},
    \label{eq:SPRL_importance_reward}
\end{equation}
where $\frac{\ContextDistribution(\context_i|\ContextDistributionParameter)}{\ContextDistribution(\context_i|\ContextDistributionParameter_{prev})}$ is an importance weight used to estimate the value of state $\CmdpCommonState_{i,0}$ in context $\context_i$ with respect to the current context distribution $\ContextDistribution(\cdot|\ContextDistributionParameter)$, as $\context_{i}$ is sampled from $\ContextDistribution(\cdot|\ContextDistributionParameter_{prev})$.

\section{Problem Formulation}
\label{section:problem}

We begin with integrating a labeling function into a CMDP to propose a labeled CMDP. We use labeled CMDPs to model long-horizon planning tasks, described via contexts.

\begin{definition}
A labeled CMDP $\LCmdp=\langle \LCmdpStates, \LCmdpActions, \LCmdpContextSpace, \LCmdpMapping \rangle$ consists of a CMDP $\Cmdp$ and a labeling function $\LCmdpLabelingFunction: \LCmdpStates \times \LCmdpActions \times \LCmdpStates \to \RMInputAlphabet$. A member of a labeled CMDP $\LCmdp$ is a labeled MDP $\LCmdpMapping(\context)=\langle \LCmdpStates, \LCmdpActions, \LCmdpTransition, \LCmdpRewardFunction, \LCmdpInitialDistribution, \LCmdpDiscount, \LCmdpLabels, \LCmdpLabelingFunction \rangle$ parameterized by a context $\context \in \LCmdpContextSpace$.
\end{definition}

A labeled MDP $\LCmdpMapping(c)$ differs from a labeled MDP $\Lmdp$, from \cref{def:LMDP}, as the former depends on a context $\context \in \LCmdpContextSpace$. 
However, every labeled MDP $\LCmdpMapping(c)$ obtained in $\LCmdp$ can use the same reward machine, that encodes the underlying task structure. Throughout this paper, we make Assumption \ref{assumption:boxshapedcontextspace} on the context space $\LCmdpContextSpace$ of a labeled CMDP $\LCmdp$. 
\begin{assumption}
There exists $\LCmdpNumberOfDimensions\in\PositiveIntegers$ and  $\LCmdpContextSpace[1],\cdots,\LCmdpContextSpace[\LCmdpNumberOfDimensions]$ such that  $\LCmdpContextSpace=\prod_{i=1}^{\LCmdpNumberOfDimensions}\LCmdpContextSpace[i]$. For $\context=(\context[1],\cdots,\context[\LCmdpNumberOfDimensions])\in\LCmdpContextSpace$, we call $\context[i]$ the $i^{th}$ context parameter of $\context$. We say $\LCmdpNumberOfDimensions$ is the dimension of the context space $\LCmdpContextSpace$, referred to as $dim(\LCmdpContextSpace)$.
\label{assumption:boxshapedcontextspace}
\end{assumption}

The assumption of a box-shaped context space $\LCmdpContextSpace$ of a labeled CMDP $\LCmdp$ allows us to establish a mapping from the transitions in the reward machine to context parameters.

\paragraph{Problem statement.}
Given a labeled CMDP $\LCmdp$, a reward machine $\RM$ that encodes the non-Markovian reward function of $\LCmdp$, and a target context distribution $\TargetContextDistribution$, we want to obtain a policy that maximizes the expected discounted return in contexts $\context$ drawn from $\TargetContextDistribution$, namely,
\begin{equation} \max_{\omega}\Expectation_{\TargetContextDistribution(\context),\LCmdpInitialDistribution(\LCmdpCommonState),\LCmdpPolicy_{\PolicyParameter}(\LCmdpCommonAction|\LCmdpCommonState,\context)}[\sum_{t=0}^{T-1}\LCmdpDiscount^t\LCmdpRewardFunction(\LCmdpHistory_t)],
\end{equation}
where $\LCmdpHistory_t=\LCmdpCommonState_0\LCmdpCommonAction_0\cdots\LCmdpCommonState_t\LCmdpCommonAction_t\LCmdpCommonState_{t+1}$ is the history at time $t$. Note that as the reward machine $\RM$ encodes the reward function $\LCmdpRewardFunction$ for any context $c$, we have $\LCmdpRewardFunction(\LCmdpHistory_t)=\RM(\LCmdpCommonLabel_1\cdots\LCmdpCommonLabel_{t+1})$ with labels $\LCmdpCommonLabel_{\tau}=\LCmdpLabelingFunction(\LCmdpCommonState_{\tau-1},\LCmdpCommonAction_{\tau-1},\LCmdpCommonState_{\tau})$ for $\tau\in[t+1].$

\section{Method}

We first present an intermediate self-paced RL algorithm by adopting the approach by \citet{icarte2018using}, which runs an RL algorithm using reward machines. Then, we discuss how contexts affect the transitions in a reward machine, and define a reward-machine-context mapping. Finally, integrating the proposed mapping into the intermediate algorithm, we develop a reward-machine-guided self-paced RL algorithm.
\subsection{Intermediate Self-Paced RL}
We construct a product contextual MDP that combines a labeled contextual MDP $\LCmdp$ and its reward machine $\RM$.
\begin{definition}
Given a labeled contextual Markov decision process $\LCmdp$ and a reward machine $\RM$, we define a product contextual MDP as the tuple $\Pmdp=\langle \PmdpStates, \PmdpActions, \PmdpContextSpace, \PmdpMapping \rangle$ that has an extended state space $\PmdpStates=\LCmdpStates \times \RMStates$, an action space $\PmdpActions$, a context space $\PmdpContextSpace$, and a mapping $\PmdpMapping$ from the context space to product MDP parameters. A member of this product contextual MDP is a tuple $\PmdpMapping(\context)=\langle \PmdpStates, \PmdpActions, \PmdpTransition, \PmdpRewardFunction, \PmdpInitialDistribution, \PmdpDiscount, \PmdpLabels, \PmdpLabelingFunction \rangle$ with a probabilistic transition function $\PmdpTransition: \PmdpStates \times \PmdpActions \times \PmdpStates \to [0,1]$, a reward function $\PmdpRewardFunction: \PmdpStates \times \PmdpActions \times \PmdpStates \to \Reals$, and an initial state distribution $\PmdpInitialDistribution: \LCmdpStates \times \{\RMInit\} \to [0,1]$. We define them as 
\begin{align}
   \PmdpTransition((\LCmdpCommonState, \RMCommonState&), \PmdpCommonAction, (\LCmdpCommonState', \RMCommonState')) \nonumber \\&= \begin{cases}
    \LCmdpTransition (\LCmdpCommonState, \PmdpCommonAction, \LCmdpCommonState') & \text{if $\RMCommonState'=\RMTransitionFunction(\RMCommonState,\PmdpLabelingFunction(\LCmdpCommonState, \PmdpCommonAction, \LCmdpCommonState'))$}; \\
    0 & \text{otherwise},
    \end{cases}  
\end{align}
\begin{equation}
    \PmdpRewardFunction((\LCmdpCommonState, \RMCommonState), \PmdpCommonAction, (\LCmdpCommonState', \RMCommonState'))=\RMOutputFunction(\RMCommonState,\PmdpLabelingFunction(\LCmdpCommonState, \PmdpCommonAction, \LCmdpCommonState'))),
\end{equation}
\begin{equation}
    \PmdpInitialDistribution(\LCmdpCommonState,\RMInit)=\LCmdpInitialDistribution(\LCmdpCommonState),
\end{equation}
where states $\LCmdpCommonState,\LCmdpCommonState' \in \LCmdpStates$ and $\RMCommonState,\RMCommonState' \in \RMStates$ come from labeled contextual MDP $\LCmdp$ and reward machine $\RM$, respectively.
\end{definition}

A trajectory of length $T$ on the product MDP $\PmdpMapping(\context)$ is $\PmdpTrajectory_i = (\PmdpCommonState_{0},\PmdpCommonAction_{0},\PmdpReward_{1},\PmdpCommonState_{1}),\cdots,$ $(\PmdpCommonState_{T-1},\PmdpCommonAction_{T-1},\PmdpReward_{i,T},\PmdpCommonState_{T})$, where $\PmdpReward_{t}=\PmdpRewardFunction(\PmdpCommonState_{t-1},\PmdpCommonAction_{t-1},\PmdpCommonState_{t})$, $t\in\{1,2,\cdots,T\}$. The intermediate self-paced RL algorithm replaces the contextual MDP trajectories with the product contextual MDP trajectories. Therefore, an RL agent can capture the temporal task structure by learning a policy via trajectories rolled out in a product contextual MDP. The intermediate self-paced RL algorithm optimizes the following objective to generate context distribution
\begin{align}
    \max_{\ContextDistributionParameter_{k}} \quad & \frac{1}{\NumberOfRollouts} \sum_{i=1}^{\NumberOfRollouts} \sum_{t=0}^{T_i-1} \CmdpDiscount^t \frac{\ContextDistribution(\context_i|\ContextDistributionParameter_{k})}{\ContextDistribution(\context_i|\ContextDistributionParameter_{k-1})} \PmdpReward_{i, t+1} \nonumber \\ 
    & - \KLCoefficient_{k} \KLDivergence(\ContextDistribution(\context|\ContextDistributionParameter_{k})\:||\:\TargetContextDistribution(\context)) \nonumber \\
    \textrm{s.t.} \quad &  \KLDivergence(\ContextDistribution(\context|\ContextDistributionParameter_{k})\:||\:\ContextDistribution(\context|\ContextDistributionParameter_{k-1})) \leq \RelativeEntropyBound,
    \label{eq:Intermediate_SPRL_objective}
\end{align}
where $\KLCoefficient_{k}$ is the regularization coefficient at the context distribution update $k$. 
Appendix~B provides the pseudocode for this algorithm.

\subsection{From Reward Machines to Contexts}
In the two-door environment (see \cref{fig:twodoor_env}), we observe that the first context parameter, i.e., the position of the first door, determines which $\Pmdp$ transitions enable the agent to pass the first door, yielding label $\{\PointMassLabelDoorOne\}$. If we change the value of the first context parameter, then we have a different set of $\Pmdp$ transitions that yield label $\{\PointMassLabelDoorOne\}$. However, this modification has no impact on the transitions that enable the agent to pass the second door, i.e., to obtain label $\{\PointMassLabelDoorTwo\}$. Taking this observation into account, we show how context parameters influence the transitions in a reward machine, then we define \emph{reward machine-context mapping} $\RMContextMapping:\RMStates\times\RMStates\to2^{\LCmdpDimensions}$, which outputs the smallest set of context parameters that determines if a transition in the reward machine happens. 

\begin{definition}
Given a product contextual MDP $\Pmdp$, we define a set $\RMMDPContextSet\subseteq \PmdpDimensions=\{1,2,\cdots,dim(\LCmdpContextSpace)\}$, as \emph{the set of identifier context parameters on a transition $(\RMCommonState, \LCmdpCommonState, \LCmdpCommonAction, \LCmdpCommonState')$} if
\begin{align}
    \forall \context,\context'\in\LCmdpContextSpace,
     \context[i]=\context'[i], \forall i\in\RMMDPContextSet \implies \nonumber\\
     \RMTransitionFunction(\RMCommonState,\LCmdpLabelingFunction(\LCmdpCommonState,\LCmdpCommonAction,\LCmdpCommonState'))= \RMTransitionFunction(\RMCommonState,\LmdpLabelingFunction_{\context'}(\LCmdpCommonState,\LCmdpCommonAction,\LCmdpCommonState')),
\end{align}
where $(\RMCommonState, \LCmdpCommonState, \LCmdpCommonAction, \LCmdpCommonState')\in\RMStates \times \LCmdpStates \times \LCmdpActions \times \LCmdpStates$. That is, $\RMMDPContextSet$ is the set of indices of the context parameters that identify the next state of the reward machine given a state $\RMCommonState$ of the reward machine and a transition $(\LCmdpCommonState,\LCmdpCommonAction,\LCmdpCommonState')$ in the labeled MDP.
\end{definition}

Notice that  $\LCmdpDimensions$ is a set of identifier context parameters for all $(\RMCommonState, \LCmdpCommonState, \LCmdpCommonAction, \LCmdpCommonState')\in\RMStates \times \LCmdpStates \times \LCmdpActions \times \LCmdpStates$.

\begin{lemma}
If $\RMMDPContextSet_1$ is a set of identifier context parameters on $(\RMCommonState, \LCmdpCommonState, \LCmdpCommonAction, \LCmdpCommonState')$, and $\RMMDPContextSet_1\subseteq\RMMDPContextSet_2\subseteq\PmdpDimensions$, then $\RMMDPContextSet_2$ is also a set of identifier context parameters on $(\RMCommonState, \LCmdpCommonState, \LCmdpCommonAction, \LCmdpCommonState')$.
\label{lemma:subsetidentifier}
\end{lemma}

\begin{proof}
Suppose $\context[i]=\context'[i]$, $\forall i \in \RMMDPContextSet_2$, then $\context[i]=\context'[i]$, $\forall i \in \RMMDPContextSet_1$ by definition.
\end{proof}

We note that if the empty set $\emptyset$ is a set of identifier context parameters on $(\RMCommonState, \LCmdpCommonState, \LCmdpCommonAction, \LCmdpCommonState')$, then the corresponding transition $\RMTransitionFunction(\RMCommonState,\LCmdpLabelingFunction(\LCmdpCommonState,\LCmdpCommonAction,\LCmdpCommonState'))$ in the reward machine does not depend on the choice of context $\context\in \LCmdpContextSpace$.

\begin{theorem}
Under Assumption \ref{assumption:boxshapedcontextspace}, $\RMMDPContextSet_1$ and $\RMMDPContextSet_2$ are sets of identifier context parameters on $(\RMCommonState,\LCmdpCommonState,\LCmdpCommonAction,\LCmdpCommonState')$ if and only if $\RMMDPContextSet_1\cap\RMMDPContextSet_2$ is a set of identifier context parameters on  $(\RMCommonState,\LCmdpCommonState,\LCmdpCommonAction,\LCmdpCommonState')$.
\label{thm:intersection}
\end{theorem}

\begin{proofsketch}\footnote{See Appendix~A for the complete proof.}
The backward statement comes from Lemma \ref{lemma:subsetidentifier}. For the forward statement, let $\context_{\RMMDPContextSet}=[\context[i]]_{i\in \RMMDPContextSet}$. Then, for any $\context,\context'\in\LCmdpContextSpace$ that satisfy $\context_{\RMMDPContextSet_1\cap\RMMDPContextSet_2}=\context'_{\RMMDPContextSet_1\cap\RMMDPContextSet_2}$, by Assumption \ref{assumption:boxshapedcontextspace} there exists $\context''$ for which $\context''_{\RMMDPContextSet_1}=\context_{\RMMDPContextSet_1}$ and $\context''_{\RMMDPContextSet_2}=\context'_{\RMMDPContextSet_2}$. Therefore, $\RMTransitionFunction(\RMCommonState,\LCmdpLabelingFunction(\LCmdpCommonState,\LCmdpCommonAction,\LCmdpCommonState'))= \RMTransitionFunction(\RMCommonState,\LmdpLabelingFunction_{\context''}(\LCmdpCommonState,\LCmdpCommonAction,\LCmdpCommonState'))= \RMTransitionFunction(\RMCommonState,\LmdpLabelingFunction_{\context'}(\LCmdpCommonState,\LCmdpCommonAction,\LCmdpCommonState'))$.
\end{proofsketch}

\begin{corollary}
Under Assumption \ref{assumption:boxshapedcontextspace}, the set $\RMMDPContextSetALL$ containing all sets of identifier context parameters on $(\RMCommonState,\LCmdpCommonState,\LCmdpCommonAction,\LCmdpCommonState')$ is closed under arbitrary unions and finite intersections.
\label{corr:arbitraryUnInter}
\end{corollary}
\begin{proof}
Lemma \ref{lemma:subsetidentifier} and Theorem \ref{thm:intersection} guarantee that $\RMMDPContextSetALL$ is closed under unions and finite intersections, respectively.
\end{proof}
Corollary \ref{corr:arbitraryUnInter} guarantees that there is a set of identifier context parameters that is contained by every set of identifier context parameters. In Definition \ref{def:minimal}, we define a mapping that provides such a set for any transition $(\RMCommonState,\LCmdpCommonState,\LCmdpCommonAction,\LCmdpCommonState')$.
\begin{definition}
Given a product contextual MDP $\Pmdp$, we define a mapping $\RMMDPContextMapping:\RMStates \times \LCmdpStates \times \LCmdpActions \times \LCmdpStates \to 2^{\PmdpDimensions}$ such that we call $\RMMDPContextMapping(\RMCommonState,\LCmdpCommonState,\LCmdpCommonAction,\LCmdpCommonState')$ ``the smallest set of identifier context parameters on $(\RMCommonState, \LCmdpCommonState, \LCmdpCommonAction, \LCmdpCommonState')$'' if $    \RMMDPContextMapping(\RMCommonState, \LCmdpCommonState, \LCmdpCommonAction, \LCmdpCommonState') = \bigcap_{\RMMDPContextSet_i\in\RMMDPContextSetALL} \RMMDPContextSet_i,$
where $\RMMDPContextSetALL$ is the set containing all possible sets of identifier context parameters on $(\RMCommonState, \LCmdpCommonState, \LCmdpCommonAction, \LCmdpCommonState')$.
\label{def:minimal}
\end{definition}
For practical applications, the design of $\RMMDPContextMapping$ is not trivial, as one needs to separately analyze every transition in a labeled contextual MDP $\LCmdp$. On the contrary, it is trivial to work with the transitions in a reward machine $\RM$, as the number of transitions in $\RM$ is smaller than the number of transitions in $\LCmdp$ in general. Therefore, we define a set of identifier context parameters for every transition in $\RM$.
\begin{definition}
Given a product contextual MDP $\Pmdp$ and the mapping $\RMMDPContextMapping$, we define a reward machine-context mapping $\RMContextMapping:\RMStates \times \RMStates \to 2^{\PmdpDimensions}$ that outputs ``a set of identifier context parameters for the transition $(\RMCommonState,\RMCommonState')$'' as
\begin{equation}
    \RMContextMapping(\RMCommonState,\RMCommonState') = \bigcup_{\mathcal{B(\RMCommonState,\RMCommonState')}} \RMMDPContextMapping(\RMCommonState,\LCmdpCommonState,\LCmdpCommonAction,\LCmdpCommonState'),
\end{equation}
\label{def:mapping}
where $ \mathcal{B(\RMCommonState,\RMCommonState')}=\{(\RMCommonState,\LCmdpCommonState,\LCmdpCommonAction,\LCmdpCommonState')\in\RMStates\times\LCmdpStates\times\LCmdpActions\times\LCmdpStates\:|\RMTransitionFunction(\RMCommonState,\LCmdpLabelingFunction(\LCmdpCommonState,\LCmdpCommonAction,\LCmdpCommonState'))=\RMCommonState'\:\text{for some}\: \context\in\LCmdpContextSpace\}.$
\end{definition}
\begin{theorem}
$\RMContextMapping(\RMCommonState, \RMCommonState')$ is the smallest set that is a set of identifier context parameters for all $(\RMCommonState,\LCmdpCommonState,\LCmdpCommonAction,\LCmdpCommonState')\in\mathcal{B(\RMCommonState,\RMCommonState')}$.
\label{thm:smallestRMContextMapping}
\end{theorem}
\begin{proof}
By Corollary \ref{corr:arbitraryUnInter}, $\RMContextMapping(\RMCommonState, \RMCommonState')$ is a set of identifier context parameters for all $(\RMCommonState,\LCmdpCommonState,\LCmdpCommonAction,\LCmdpCommonState')\in\mathcal{B(\RMCommonState,\RMCommonState')}$. 
Also, a set $\mathcal{U}$ that is guaranteed to be a set of identifier context parameters for all $(\RMCommonState,\LCmdpCommonState,\LCmdpCommonAction,\LCmdpCommonState')\in\mathcal{B(\RMCommonState,\RMCommonState')}$ must contain $\RMContextMapping(\RMCommonState, \RMCommonState')$ by construction. Then, $|\mathcal{U}|\geq|\RMContextMapping(\RMCommonState, \RMCommonState')|$.
\end{proof}
An expert designs mapping $\RMContextMapping$ by asking questions about the task structure. 
For instance, for the transition ($\RMCommonState_1,\RMCommonState_5$) in \cref{fig:RM_two_door}, the expert should ask: 
Is there a transition ($\LCmdpCommonState,\LCmdpCommonAction,\LCmdpCommonState’$) in the labeled contextual MDP $\LCmdp$ such that it causes the agent to hit the wall for some context $\context$ but lets the agent pass through the door, i.e., ($\RMCommonState_1,\RMCommonState_2$), for a different context $\context’$? 
The idea is to find the context parameters $i\in\LCmdpDimensions$ for which a change of value, e.g. $\context[i] \neq \context’[i]$, prevents a transition $(\RMCommonState,\RMCommonState’)$ in the reward machine from happening. 
For ($\RMCommonState_1,\RMCommonState_5$), the mapping outputs the first context parameter, $\RMContextMapping(\RMCommonState_1,\RMCommonState_5)=\{1\}$, as the identifier, since it determines the position of the first door.
In short, when the agent is in the second room and moves into the first door/wall with an upward action, then the position of the first door determines whether it moves into the door or the wall. 
Here, the position of the second door does not identify which transition will happen.

\subsection{Reward-Machine-Guided, Self-Paced Reinforcement Learning}
\citet{NEURIPS2020_68a97503}'s self-paced RL algorithm uses an importance weight in (\ref{eq:SPRL_importance_reward}) as the ratio between probabilities of a context with respect to the current and previous contexts distributions. In other words, the algorithm assumes that every context parameter has an effect on the reward that an environment interaction yields. On the other hand, by \cref{thm:smallestRMContextMapping}, a reward machine-context mapping  $\RMContextMapping$ outputs the smallest set of identifier context parameters for a transition $(\RMCommonState, \RMCommonState')$ in the reward machine $\RM$. Therefore, we can remove the naive assumption of \citet{NEURIPS2020_68a97503} and use the context parameters that the mapping provides to compute the importance weight of a reward received in a transition $(\RMCommonState, \RMCommonState')$. We achieve this by utilizing the marginal context distributions for the context parameters in the set $\RMContextMapping(\RMCommonState, \RMCommonState')$ as
$\frac{1}{\NumberOfRollouts} \sum_{i=1}^{\NumberOfRollouts} \sum_{t=0}^{T_i-1} \CmdpDiscount^t \frac{\ContextDistribution_{\RMContextMappingCommonOutput_t}(\context_i|\ContextDistributionParameter_{k})}{\ContextDistribution_{\RMContextMappingCommonOutput_t}(\context_i|\ContextDistributionParameter_{k-1})} \PmdpReward_{i, t+1},$
where $\PmdpReward_{i,t}=\PmdpRewardFunction[\context_i](\PmdpCommonState_{i,t-1},\PmdpCommonAction_{i,t-1},\PmdpCommonState_{i,t})$ and $\RMContextMappingCommonOutput_t=\RMContextMapping(\RMCommonState_t,\RMCommonState_{t+1})$. 
Here, we introduce  $\ContextDistribution_{\RMContextMappingCommonOutput_t}(\cdot|\ContextDistributionParameter_{k})$ and $\ContextDistribution_{\RMContextMappingCommonOutput_t}(\cdot|\ContextDistributionParameter_{k-1})$, that are the current and previous marginal context distributions, where the marginal variables are the identifier context parameters in $\RMContextMappingCommonOutput_t$, respectively. 
We note that for the case $\RMContextMappingCommonOutput_t=\emptyset$, we assign $\frac{\ContextDistribution_{\RMContextMappingCommonOutput_t}(\context_i|\ContextDistributionParameter_{k})}{\ContextDistribution_{\RMContextMappingCommonOutput_t}(\context_i|\ContextDistributionParameter_{k-1})}=1$. 
Consequently, the reward-machine-guided, self-paced RL algorithm optimizes the following objective for context distribution updates, namely,
\begin{align}
    \max_{\ContextDistributionParameter_{k}} \quad & \frac{1}{\NumberOfRollouts} \sum_{i=1}^{\NumberOfRollouts} \sum_{t=0}^{T_i-1} \CmdpDiscount^t \frac{\ContextDistribution_{\RMContextMappingCommonOutput_t}(\context_i|\ContextDistributionParameter_{k})}{\ContextDistribution_{\RMContextMappingCommonOutput_t}(\context_i|\ContextDistributionParameter_{k-1})} \PmdpReward_{i, t+1} \nonumber \\ 
    & - \KLCoefficient_{k} \KLDivergence(\ContextDistribution(\context|\ContextDistributionParameter_{k})\:||\:\TargetContextDistribution(\context)) \nonumber \\
    \textrm{s.t.} \quad &  \KLDivergence(\ContextDistribution(\context|\ContextDistributionParameter_{k})\:||\:\ContextDistribution(\context|\ContextDistributionParameter_{k-1})) \leq \RelativeEntropyBound,
    \label{eq:RM_guided_SPRL_objective}
\end{align}
where $\KLCoefficient_{k}$ is the regularization coefficient at the context distribution update $k$. Similar to the intermediate self-paced RL, the reward-machine-guided, self-paced RL algorithm runs on a product contextual MDP $\Pmdp$, as well. We outline the complete algorithm in \cref{alg:RM_guided_SPRL}. Lines 3-5 update the policy $\LCmdpPolicy$ using trajectories in the sampled contexts via an RL algorithm $\RLAlgorithm$. Line 6 generates context distributions.
\begin{algorithm}[tbp]
\caption{Reward-Machine-Guided, Self-Paced RL}
\label{alg:RM_guided_SPRL}
\textbf{Input}: Product MDP $\Pmdp$, reward machine-context mapping $\RMContextMapping$, target context distribution $\TargetContextDistribution$, initial context distribution $\ContextDistribution(\cdot|\ContextDistributionParameter_{0})$,\\
\textbf{Parameter}: KL penalty proportion $\KLPenaltyProportion$, relative entropy bound $\RelativeEntropyBound$,  KL penalty offset offset $\KLPenaltyOffset$, number $\NumberOfIterations$ of iterations, number $\NumberOfRollouts$ of rollouts\\
\textbf{Output}: Final policy $\PmdpPolicy_{\PolicyParameter_{\NumberOfIterations}}$
\begin{algorithmic}[1] 
    \STATE Initialize policy $\PmdpPolicy_{\PolicyParameter_{0}}$.
    \FOR{$k=1$ to $\NumberOfIterations$}
    \STATE $\context_i\sim\ContextDistribution(\context|\ContextDistributionParameter_{k-1})$, $i\in[\NumberOfRollouts]$,
    \COMMENT{sample contexts}
    \STATE $\SPRLTrajectorySet_{k}\leftarrow\{(\context_i, \PmdpTrajectory_i) | \PmdpTrajectory_i = (\PmdpCommonState_{i,0},\PmdpCommonAction_{i,0},\PmdpReward_{i,1},\PmdpCommonState_{i,1}),\cdots,$ $(\PmdpCommonState_{i,T_i-1},\PmdpCommonAction_{i,T_i-1},\PmdpReward_{i,T_i},\PmdpCommonState_{i,T_i}),i\in[\NumberOfRollouts]\}$,
    \COMMENT{collect trajectories}
    \STATE $\PmdpPolicy_{\PolicyParameter_{k}}\leftarrow \Psi(\SPRLTrajectorySet_{k},\PmdpPolicy_{\PolicyParameter_{k-1}})$
    \COMMENT{update policy with RL algorithm $\RLAlgorithm$}
    \STATE Compute next context distribution parameter $\ContextDistributionParameter_{k}$ by optimizing (\ref{eq:RM_guided_SPRL_objective}) with \\
    $\KLCoefficient_{k}=\begin{cases}
        $0$ & \text{if $k \leq \KLPenaltyOffset$}; \\
        \mathsf{B}(\ContextDistributionParameter_{k-1},\SPRLTrajectorySet_{k}) & \text{otherwise},
        \end{cases}$\\
    where $\mathsf{B}(\ContextDistributionParameter_{k-1},\SPRLTrajectorySet_{k}) = \KLPenaltyProportion \frac{\max{(0, \frac{1}{\NumberOfRollouts}\sum_{i=1}^{\NumberOfRollouts}\sum_{t=1}^{T_i} \PmdpDiscount^t \PmdpReward_{i, t})}}{\KLDivergence(\ContextDistribution(\context|\ContextDistributionParameter_{k-1})||\TargetContextDistribution(\context))}$.
    \ENDFOR
    \STATE \textbf{return} $\PmdpPolicy_{\PolicyParameter_{\NumberOfIterations}}$
\end{algorithmic}
\end{algorithm}

\section{Empirical Results}

We evaluate the proposed RM-guided SPRL and Intermediate SPRL with three state-of-the-art automated curriculum generation methods\footnote{See \url{https://github.com/cevahir-koprulu/rm-guided-sprl} to access the code repository of this work.}: SPDL \citep{NEURIPS2020_68a97503}, GoalGAN \citep{florensa2018automatic}, and ALP-GMM \citep{portelas2020teacher}. We also include two baseline approaches: Default, which draws samples from the target context distribution without generating a curriculum, and Default*, which extends Default by running the RL algorithm on a product contextual MDP, hence we observe the effect of capturing temporal abstractions. Appendix~B includes more details.

\paragraph{Two-door environment.} The two-door environment is a variation of the point-mass environment \cite{klink2020self_SPCRL,NEURIPS2020_68a97503,klink2021probabilistic,klink2022curriculum} with a temporal structure. Similar to commonly studied domains such as the office world, craft world, and water world \cite{icarte2018using,camacho2019ltl,icarte2022reward}, the two-door environment has discrete state and action spaces as a 40-by-40 grid world. The context space $\PmdpContextSpace=[-4,4]^2$ includes all available horizontal positions for two doors. The target context distribution $\TargetContextDistribution(\context)$ is a normal distribution $\mathcal{N}(\mu,\Sigma)$, where $\mu=(\mu_1, \mu_2)=(2,2)$ and $\Sigma=diag((\sigma_1^2, \sigma_2^2))=\boldsymbol{I}_2$, where $\boldsymbol{I}_2$ is the identity matrix. 

\cref{fig:twodoordiscrete2D_curriculum} demonstrates how the curricula generated by each method evolve over the training. Here, we exclude GoalGAN, Default*, and Default, since they do not have a notion of a target context distribution. RM-guided SPRL produces sequences of context distributions that vary less over the same curriculum updates and converges faster, by one-fourth, than Intermediate SPRL and SPDL. Table \ref{tab:curricula_variance} demonstrates the average variance of the statistics of the context distributions, i.e., mean and variance of normal distributions, generated until convergence to the target distribution. RM-guided SPRL has the lowest variances for all statistics, up-to four orders of magnitude, with statistical significance $p<0.001$. Intuitively, guiding the curriculum generation process via a reward-machine-context mapping $\RMContextMapping$ allows for avoiding redundant exploration of the curriculum space.

\begin{figure}[t]
\centering
    \begin{subfigure}{.8\linewidth}
    \centering
    \includegraphics[width=\linewidth]{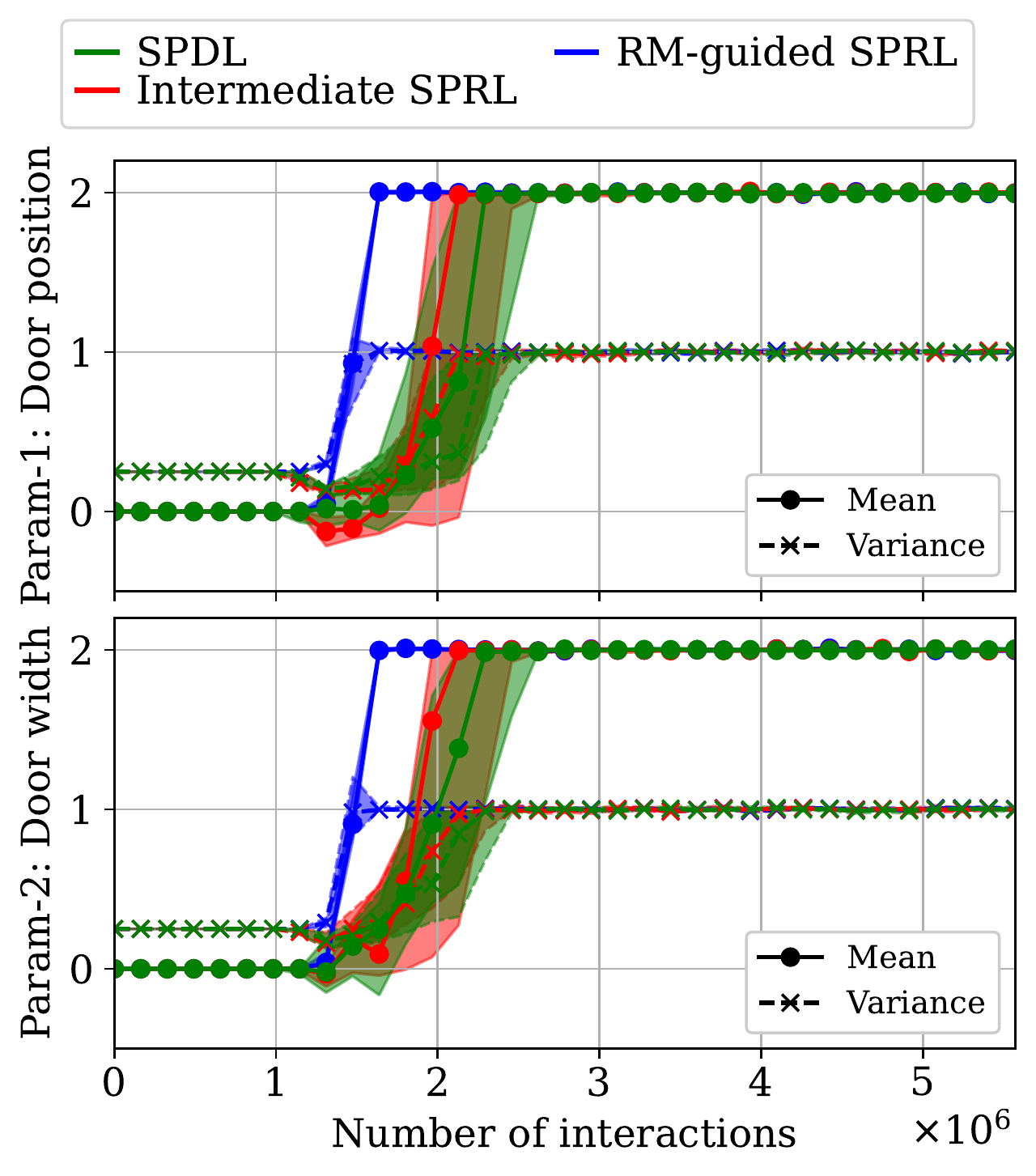}
    \caption{Progression of the statistics (mean and variance) of context distributions generated in the curriculum. 
    }
    \label{fig:twodoordiscrete2D_curriculum}
    \end{subfigure}
~
    \begin{subfigure}{.95\linewidth}
    \centering
    \includegraphics[width=\linewidth]{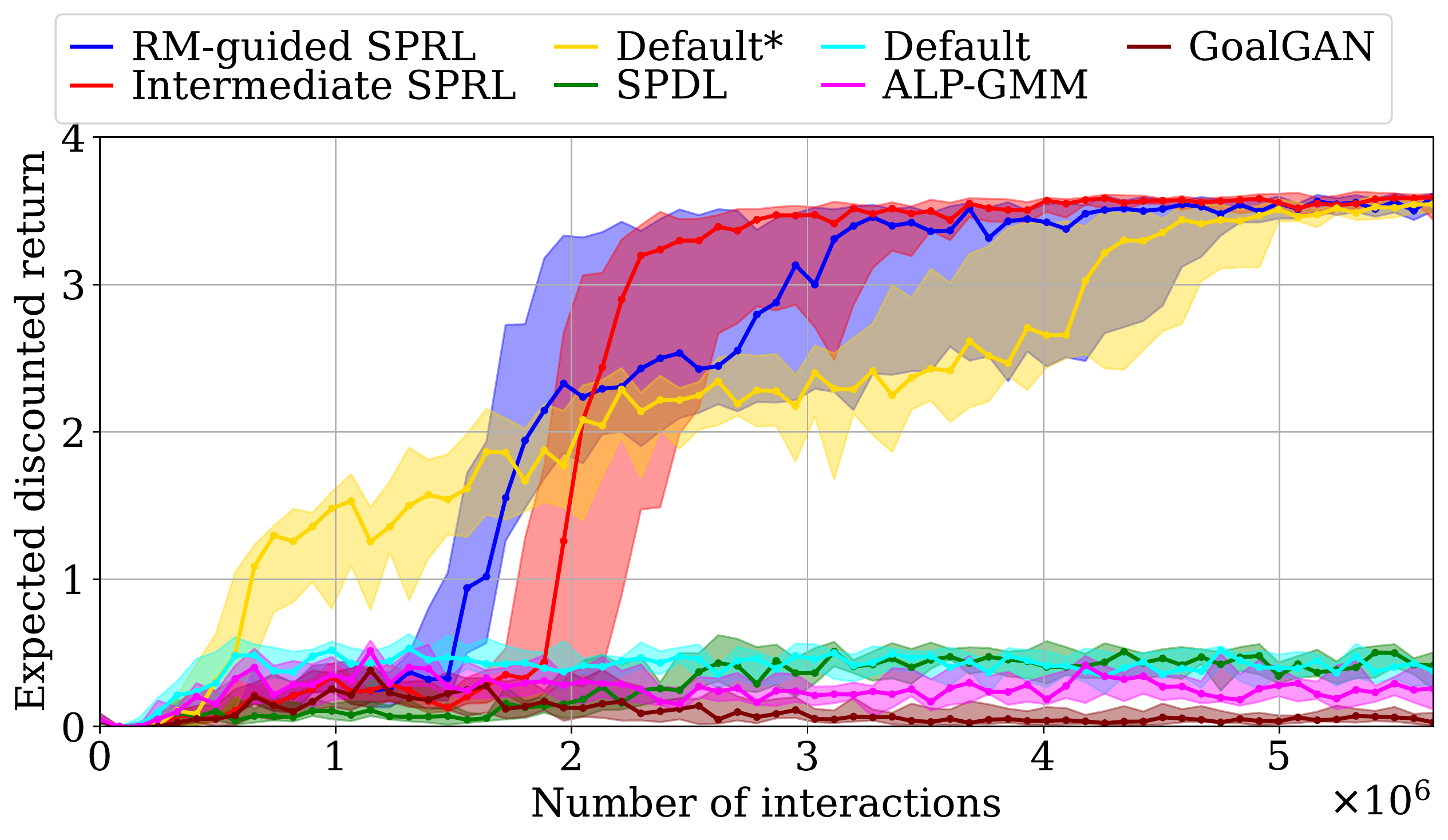}
    \caption{Progression of the expected discounted return with respect to the target context distribution.
    }
    \label{fig:twodoordiscrete2D_performance}
    \end{subfigure}
\caption{Two-door environment: Progression of the curricula and performance during the training. Shaded regions cover the quartiles. Bold lines indicate the median values.}
\end{figure}

\cref{fig:twodoordiscrete2D_performance} shows the expected discounted return progression. Although Default* obtains a higher return early on, it lags behind RM-guided SPRL and Intermediate SPRL since it does not generate a curriculum, but samples from the target context distribution directly. As we set $\KLPenaltyOffset$ to 70 for all self-paced algorithms, the agent draws easy contexts from the initial context distribution $\ContextDistribution(\cdot|\ContextDistributionParameter_{0})$. RM-guided SPRL surpasses Default* quickly, but the agent seems to get stuck in the final phase, finding the goal. Intermediate SPRL do not experience this as its curricula converge later (see red lines in \cref{fig:twodoordiscrete2D_curriculum}). Other approaches cannot learn a policy that accomplishes the task because they do not capture the temporal structure described by the reward machine.

\begin{table}[t]
\centering
\caption{Curricula variance: Average variance of the statistics of the context distributions generated by self-paced RL methods in three case studies. $\mu_i$ and $\sigma_i^2$ correspond to the statistics of a normal distribution, i.e., the mean and the variance for the $i^{th}$ context parameter, respectively. The variances that are highlighted in \textrm{bold} are significantly better results (lower variance of a statistic) with $p<0.001$.}
\begin{tabular}{ccccc}\toprule
 &Stat & RM-guided & Intermediate & SPDL \\ \midrule
   \multirow{4}{*}{\rotatebox{90}{Two-door}}
   &$\mu_{1}$      & $\boldsymbol{4.09 \cdot 10^{-3}}$  & $2.06 \cdot 10^{-1}$  & $1.21  \cdot 10^{-1}$ \\
   &$\mu_{2}$      & $\boldsymbol{4.61 \cdot 10^{-3}}$  & $1.69 \cdot 10^{-1}$  & $8,78  \cdot 10^{-2}$ \\
   &$\sigma_{1}^2$ & $\boldsymbol{2.23 \cdot 10^{-3}}$  & $3.75 \cdot 10^{-2}$  & $2.46  \cdot 10^{-2}$ \\
   &$\sigma_{2}^2$ & $\boldsymbol{2.31 \cdot 10^{-3}}$  & $2.59 \cdot 10^{-2}$  & $2.42  \cdot 10^{-2}$ \\ \midrule

   \multirow{4}{*}{\rotatebox{90}{Swimmer}}
   &$\mu_{1}$      & $3.02 \cdot 10^{-4}$  & $8.18 \cdot 10^{-4}$  & $2.37 \cdot 10^{-2}$ \\
   &$\mu_{2}$      & $5.07 \cdot 10^{-4}$  & $1.13 \cdot 10^{-4}$  & $2.90 \cdot 10^{-2}$ \\
   &$\sigma_{1}^2$ & $2.83 \cdot 10^{-6}$  & $9.55 \cdot 10^{-7}$  & $8.40 \cdot 10^{-6}$ \\
   &$\sigma_{2}^2$ & $2.68 \cdot 10^{-6}$  & $2.77 \cdot 10^{-6}$  & $6.16 \cdot 10^{-5}$ \\ \midrule
    
   \multirow{6}{*}{\rotatebox{90}{HalfCheetah}}
   &$\mu_{1}$      & $\boldsymbol{2.90 \cdot 10^{-2}}$  & $6.23 \cdot 10^{-1}$  & - \\
   &$\mu_{2}$      & $\boldsymbol{2.39 \cdot 10^{-2}}$  & $9.84 \cdot 10^{-1}$  & - \\
   &$\mu_{3}$      & $\boldsymbol{9.85 \cdot 10^{-2}}$  & $1.27 \cdot 10^{-1}$  & - \\
   &$\sigma_{1}^2$ & $\boldsymbol{9.13 \cdot 10^{-4}}$  & $5.80 \cdot 10^{-3}$  & - \\
   &$\sigma_{2}^2$ & $1.92             \cdot 10^{-3}$   & $1.95 \cdot 10^{-4}$  & - \\
   &$\sigma_{3}^2$ & $\boldsymbol{2.86 \cdot 10^{-3}}$  & $5.35 \cdot 10^{-2}$  & - \\ \bottomrule
\end{tabular}
\label{tab:curricula_variance}
\end{table}

\paragraph{Customized Swimmer-v3 environment.}

\begin{figure}[t]
    \centering
    \begin{tikzpicture}[
    shorten >=.2pt,node distance=2.1
    cm,on grid,auto, align=center,
    state/.style={circle, draw, minimum size=.5cm}
] 
        \node[state, initial] (q0) {$\RMCommonState_0$};
        \node[state, right of=q0] (q1) {$\RMCommonState_1$};
        \node[state, accepting, right of=q1] (q2) {$\RMCommonState_2$};
        
        \draw (q0) edge node {\tt $(\SwimmerFlagOne, 100)$}(q1);
        \draw (q1) edge node{\tt $(\SwimmerFlagTwo , 1000)$}(q2);

        \draw (q0) edge[loop above] node {\tt$(\neg \SwimmerFlagOne, \text{CP})$} (q0);
        \draw (q1) edge[loop above] node {\tt$(\neg \SwimmerFlagTwo, \text{CP})$} (q1);
        \draw (q2) edge[loop above] node {\tt$(\top, 0)$} (q2);
    \end{tikzpicture}
        \caption{Reward machine of the Swimmer environment}
        \label{fig:RM_swimmer}
\end{figure}
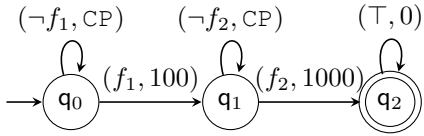

\begin{figure}[t]
\centering
\includegraphics[width=.95\linewidth]{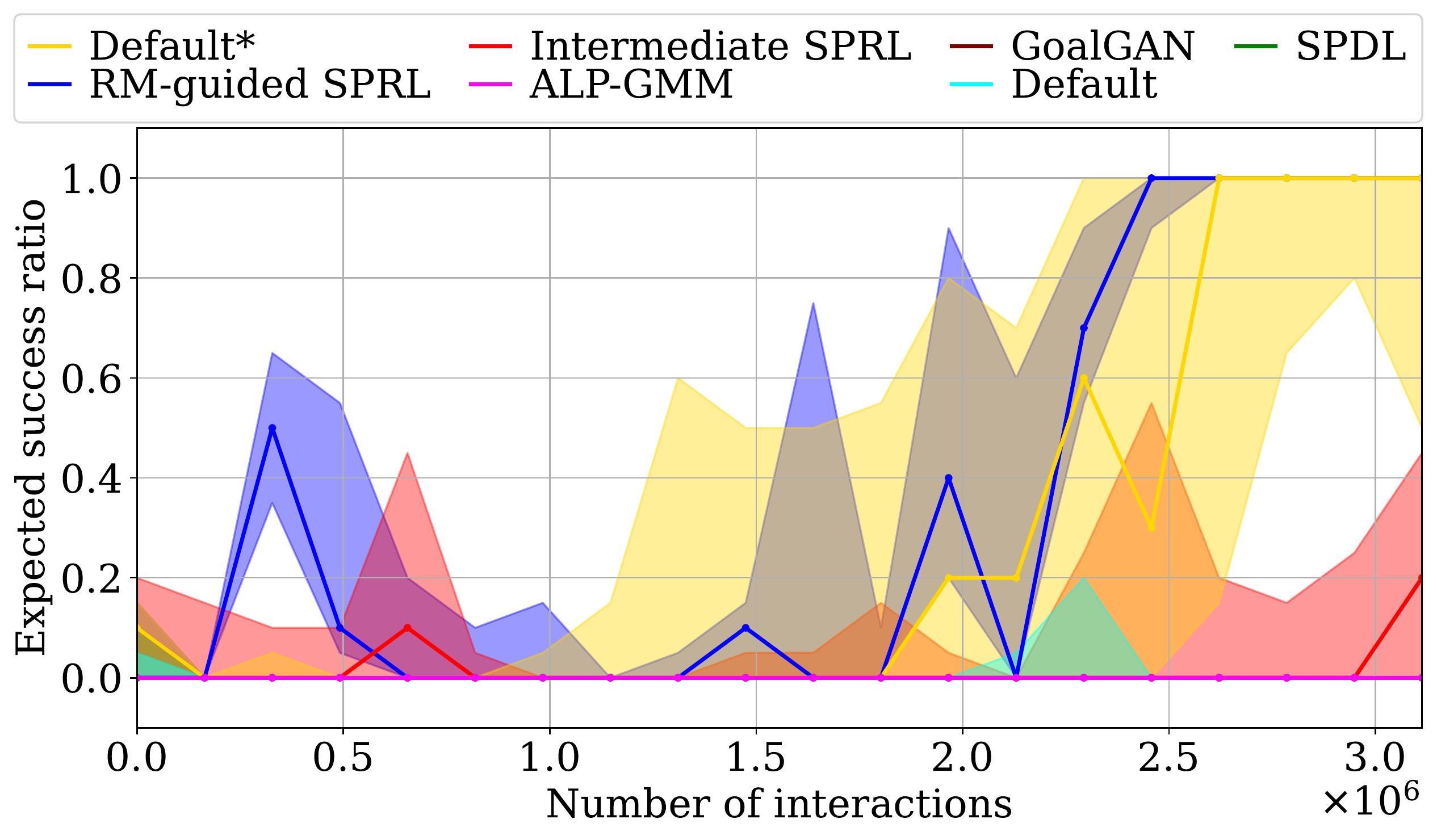}
\caption{Swimmer-v3 environment: Progression of the successful episodes ratio in contexts drawn from the target context distribution over curriculum updates.}
\label{fig:swimmer2D_performance}
\end{figure}

We customize the Swimmer-v3 environment \cite{1606.01540} by adding two flags, namely checkpoints, to the left, $\SwimmerFlagOne$, and the right, $\SwimmerFlagTwo$, of the initial position of the swimmer. In comparison to the two-door environment, customized Swimmer-v3 has continuous state (8-dimensional) and action (2-dimensional) spaces. However, the underlying task is the simplest in our experimental setup, with a reward machine of 3 states (see \cref{fig:RM_swimmer}): The swimmer has to visit flag 2 $\SwimmerFlagTwo$ first, and then flag 1, obtaining a reward of 100 and 1000, respectively. Inspired by \citet{icarte2022reward}, we use a control penalty, noted as $\text{CP}$, for rewards received by the agent following the self-loop transitions in the reward machine, to discourage the agent from applying large forces to the joints. The context space is 2-dimensional and determines the positions of the flags: $\LCmdpContextSpace_2=[-0.6,0]\times[1,1.6]$. The target context distribution is $\mathcal{N}((0.6,1.6,\boldsymbol{I}_3\cdot1.6\cdot 10^{-7})$.

\cref{fig:swimmer2D_performance} shows that only RM-guided-SPRL and Default* achieve 100\% success ratio in median. RM-guided SPRL converges faster, and is more reliable as the quartiles converge before the training ends, as well. Default*'s performance evidence that this task does not require a curriculum as much as the two-door environment. Intermediate SPRL's failure supports this argument, as it cannot achieve a success ratio of more than 20\%, in the median. The other algorithms, again, fail to accomplish the task. \cref{tab:curricula_variance} indicates that the curricula variance of RM-guided SPRL is not significantly different than Intermediate SPRL. Nevertheless, RM-guided SPRL is reliable as it is 100\% successful (median) while avoiding the redundant exploration of the curriculum space.

\paragraph{Customized HalfCheetah-v3 environment.}
\begin{figure}[t]
    \centering
    \includegraphics[width=.95\linewidth]{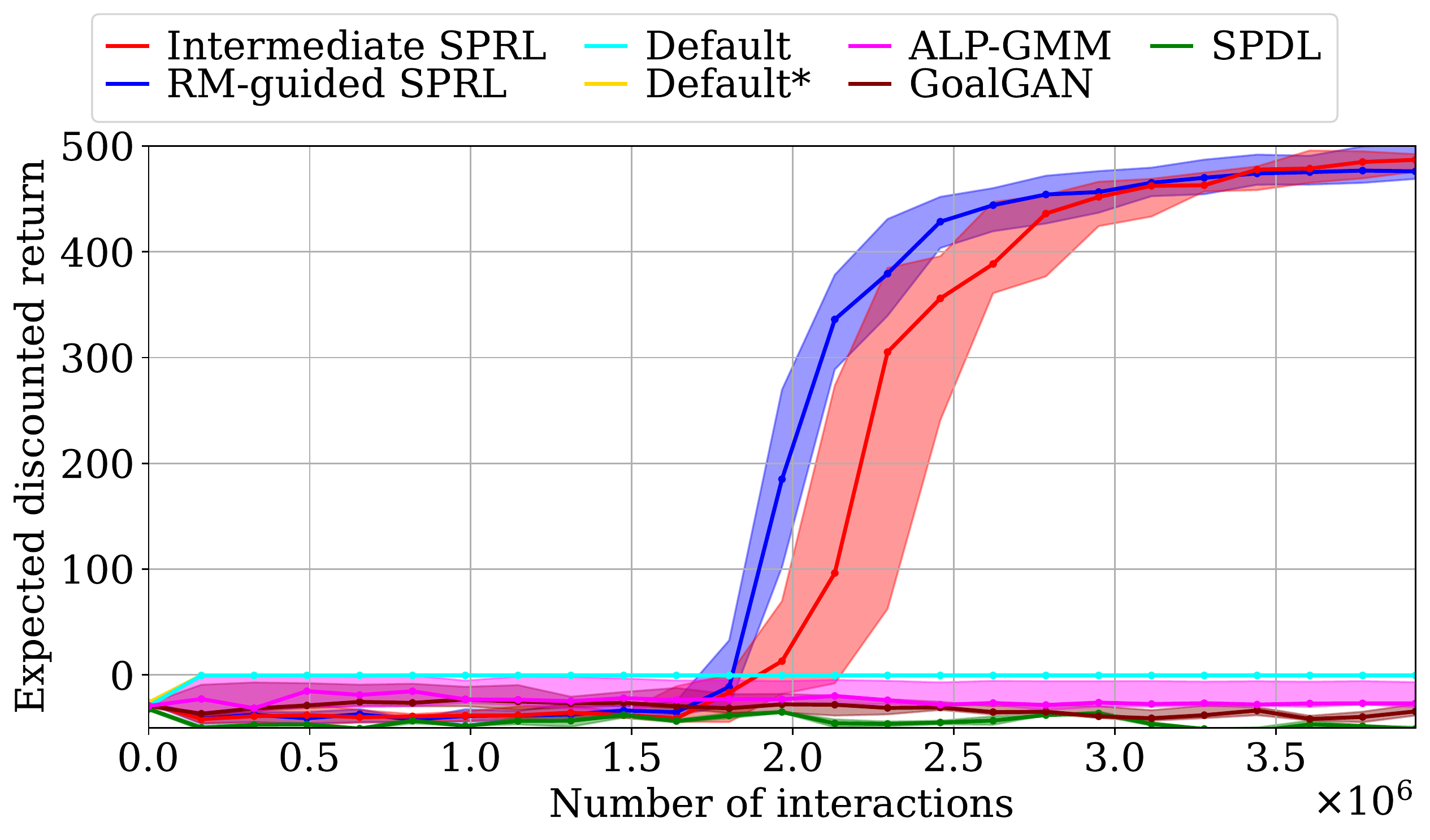}
    \caption{HalfCheetah-v3 environment: Progression of the expected discounted return with respect to the target context distribution over curriculum updates.}
    \label{fig:halfcheetah3D_perf}
\end{figure}

We also customize the HalfCheetah-v3 environment \cite{1606.01540} by adding three flags, $\CheetahFlagOne$, $\CheetahFlagTwo$, and $\CheetahFlagThree$, ordered in ascending distance to the right of the cheetah's initial position. The reward machine in \cref{fig:RM_half_cheetah} describes the following task: The cheetah has to visit flags 1 and 2, then go back to flag 1, and finally pass flag 3. The underlying task requires the cheetah to change direction 3 times. The original HalfCheetah-v3  \cite{1606.01540} and its variation in \cite{icarte2022reward} are single-task environments and reward the cheetah for running forward, only. In comparison, customized HalfCheetah-v3 has a running backward subtask (the transition ($\RMCommonState_2,\RMCommonState_3$) in the reward machine in Figure \cref{fig:RM_half_cheetah}), which is challenging for the agent. Similar to customized Swimmer-v3, customized HalfCheetah-v3 has continuous state (17-dimensional) and action (6-dimensional) spaces. The 3-dimensional context space determines the flag positions: $\LCmdpContextSpace_3=[0.5,4]\times[2,7]\times[3.5,10]$. The target context distribution is $\mathcal{N}((4,7,10),\boldsymbol{I}_3\cdot1.6\cdot 10^{-7})$.

\cref{fig:halfcheetah3D_perf} shows that RM-guided SPRL is the only algorithm that can learn a policy that accomplishes the target contexts in every independent training run. Intermediate SPRL fails in one run, where the generated curricula cannot converge to the target context distribution during the training (see Appendix~B). SPDL suffers from a similar issue because they both obtain a negative expected discounted return, which sets $\KLCoefficient_{k}$ in (\ref{eq:RM_guided_SPRL_objective}) to zero, hence the generated curricula do not approach, even diverge from, the target context distribution. Default*, Default, GoalGAN, and ALP-GMM are unsuccessful in this domain. In Table \ref{tab:curricula_variance}, we exclude SPDL, as none of its curricula converge to the target before the training ends. Similar to Case-2, RM-guided SPRL generates curricula whose variance is significantly smaller with $p<0.001$.

\begin{figure}[t]
    \centering
    \begin{tikzpicture}[
    shorten >=.1pt,node distance=2.
    cm,on grid,auto, align=center,
    state/.style={circle, draw, minimum size=.2cm}
] 
        \node[state, initial] (q0) {$\RMCommonState_0$};
        \node[state, right of=q0] (q1) {$\RMCommonState_1$};
        \node[state, right of=q1] (q2) {$\RMCommonState_2$};
        \node[state, below of=q2] (q3) {$\RMCommonState_3$};
        \node[state, accepting, left of=q3] (q4) {$\RMCommonState_4$};
        
        \draw (q0) edge node {\tt $(\CheetahFlagOne, 10)$}(q1);
        \draw (q1) edge node{\tt $(\CheetahFlagTwo , 100)$}(q2);
        \draw (q2) edge node {\tt $(\CheetahFlagOne, 500)$}(q3);
        \draw (q3) edge node {\tt $(\CheetahFlagThree, 1000)$}(q4);

        \draw (q0) edge[loop above] node {\tt$(\neg \CheetahFlagOne, \text{CP})$} (q0);
        \draw (q1) edge[loop above] node {\tt$(\neg \CheetahFlagTwo, \text{CP})$} (q1);
        \draw (q2) edge[loop above] node {\tt$(\neg \CheetahFlagOne, \text{CP})$} (q2);
        \draw (q3) edge[loop right] node {\tt$(\neg \CheetahFlagThree, \text{CP})$} (q3);
        \draw (q4) edge[loop above] node {\tt$(\top, 0)$} (q4);
    \end{tikzpicture}
        \caption{Reward machine of the HalfCheetah environment}
        \label{fig:RM_half_cheetah}
\end{figure}
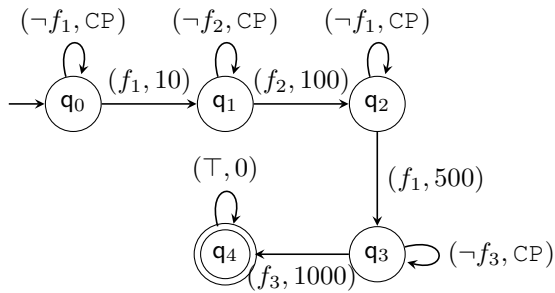

\begin{figure}[t]
\centering
    \begin{subfigure}{.45\linewidth}
    \centering
    \includegraphics[width=\linewidth]{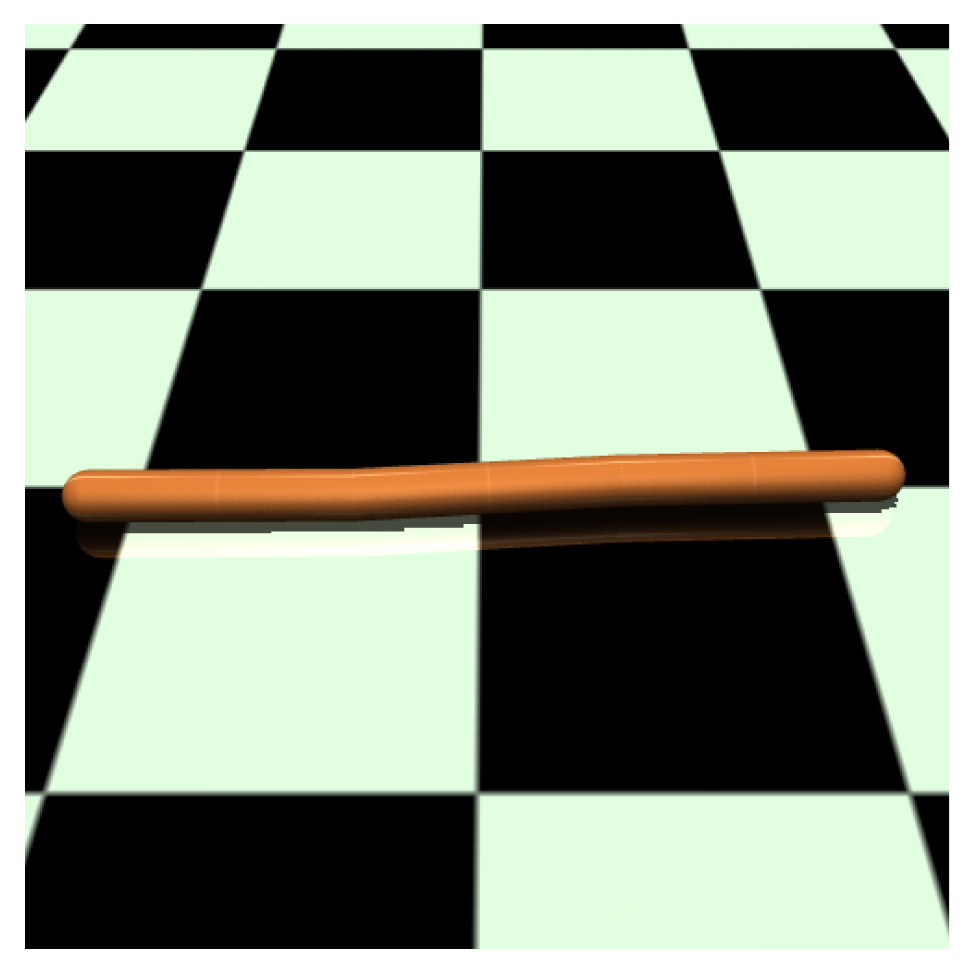}
    \caption{Swimmer-v3}
    \label{fig:Swimmer-v3}
    \end{subfigure}
~
    \begin{subfigure}{.45\linewidth}
    \centering
    \includegraphics[width=\linewidth]{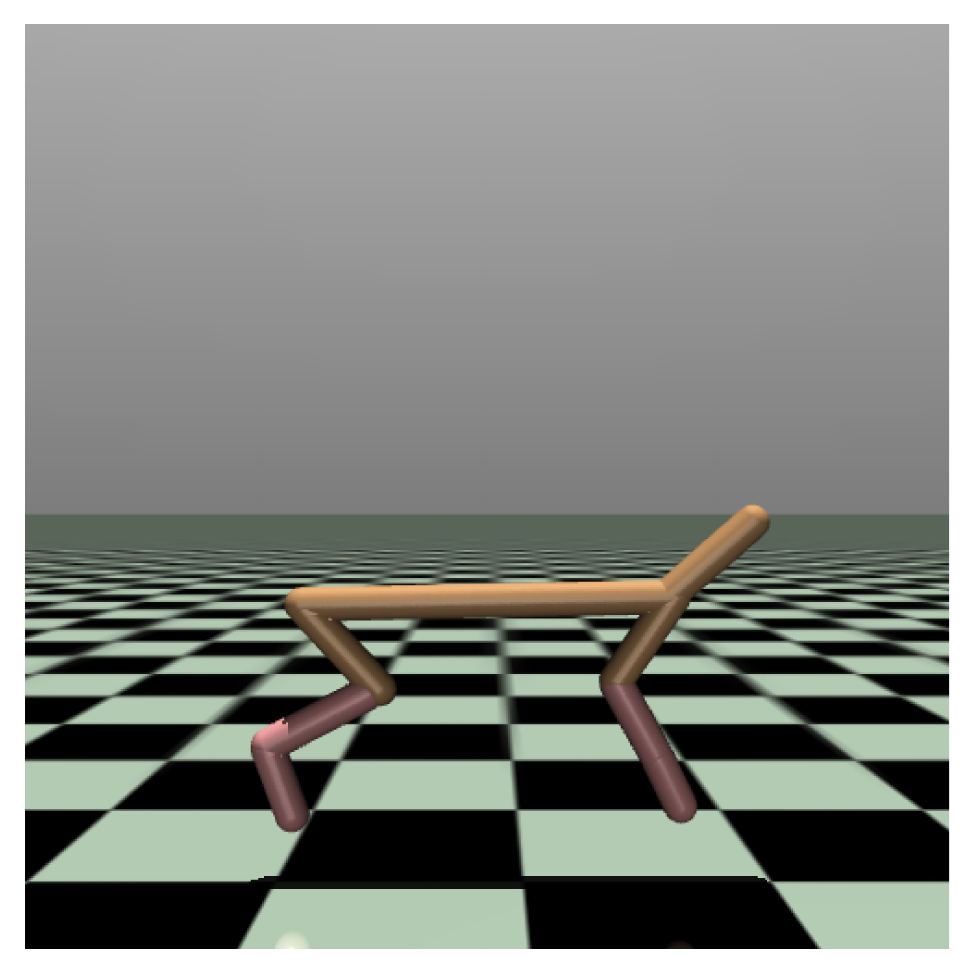}
    \caption{HalfCheetah-v3}
    \label{fig:HalfChetah-v3}
    \end{subfigure}
\caption{Images from Customized Swimmer-v3 and Customized HalfCheetah-v3 environments.}
\end{figure}

\section{Conclusions}

We propose two self-paced RL algorithms that exploit the high-level structural knowledge about long-horizon planning tasks via reward machines. First, we present an intermediate self-paced RL algorithm that uses reward machines to update the policy and value functions of an RL agent. Then, we establish a mapping, called reward-machine-context mapping, that, given a transition in the reward machine, outputs the smallest set of identifier context parameters that determines whether the transition occurs or not. Lastly, we develop a reward-machine-guided, self-paced RL algorithm that builds on the intermediate algorithm and navigates the automated curriculum generation via reward-machine-context mapping. We evaluate the proposed algorithms in three domains. We empirically show that existing approaches fail to accomplish the given long-horizon planning tasks, whereas the proposed algorithms can capture the temporal structure of such tasks. Compared to the intermediate algorithm, the reward-machine-guided, self-paced RL algorithm is more reliable, as it achieves successful completion of the task in every use case, and it also reduces curricula variance by up to four orders of magnitude. 

\paragraph{Limitations.} The limitations come from the self-paced RL algorithm used in the proposed approaches, the assumption of a priori available reward machine, and task knowledge to construct a reward-machine-context mapping: 
\begin{enumerate*}[label=\arabic*)]
    \item Intermediate SPRL and RM-guided SPRL employ a self-paced RL algorithm, SPDL \cite{NEURIPS2020_68a97503}, which uses a parametric family of context distributions to generate a curriculum. Similar to \cite{NEURIPS2020_68a97503,klink2021probabilistic}, we study Gaussian context distributions. Hence, SPDL does not address settings with arbitrary target context distributions. 
    \item We focus on long-horizon planning tasks with a priori available reward machines. The proposed approaches require a reward machine to construct a product contextual MDP, which captures the temporal task structure.
    \item Task knowledge about the connection between the reward machine and the context space enables the design of a reward-machine-context mapping. Unless such knowledge is available, RM-guided SPRL and Intermediate SPRL become equivalent, as the latter do not utilize the mapping.
\end{enumerate*} 

\paragraph{Future Work.} Taking into account the limitations of the proposed approaches, we will study how to infer a reward machine and a reward-machine-context mapping of a domain online to remove the need for a priori available task knowledge. In addition, we will extend RM-guided SPRL to address arbitrary context distributions, which \citet{klink2022curriculum} studies without integrating high-level structural task knowledge.

\balance


\begin{acknowledgements} 
This work is supported by the Office of Naval Research (ONR) under grant number N00014-22-1-2254, and the National Science Foundation (NSF) under grant number 1646522.
\end{acknowledgements}

\bibliography{koprulu_587}

\begin{thebibliography}{39}
\providecommand{\natexlab}[1]{#1}
\providecommand{\url}[1]{\texttt{#1}}
\expandafter\ifx\csname urlstyle\endcsname\relax
  \providecommand{\doi}[1]{doi: #1}\else
  \providecommand{\doi}{doi: \begingroup \urlstyle{rm}\Url}\fi

\bibitem[Andrychowicz et~al.(2017)Andrychowicz, Wolski, Ray, Schneider, Fong,
  Welinder, McGrew, Tobin, Pieter~Abbeel, and Zaremba]{NIPS2017_453fadbd}
Marcin Andrychowicz, Filip Wolski, Alex Ray, Jonas Schneider, Rachel Fong,
  Peter Welinder, Bob McGrew, Josh Tobin, OpenAI Pieter~Abbeel, and Wojciech
  Zaremba.
\newblock Hindsight experience replay.
\newblock In \emph{NeurIPS}, 2017.

\bibitem[Bacchus et~al.(1996)Bacchus, Boutilier, and
  Grove]{bacchus1996rewarding}
Fahiem Bacchus, Craig Boutilier, and Adam Grove.
\newblock Rewarding behaviors.
\newblock In \emph{National Conference on Artificial Intelligence}, pages
  1160--1167, 1996.

\bibitem[Baranes and Oudeyer(2010)]{baranes2010intrinsically}
Adrien Baranes and Pierre-Yves Oudeyer.
\newblock Intrinsically motivated goal exploration for active motor learning in
  robots: A case study.
\newblock In \emph{IROS}, pages 1766--1773, 2010.

\bibitem[Bengio et~al.(2009)Bengio, Louradour, Collobert, and
  Weston]{bengio2009curriculum}
Yoshua Bengio, J{\'e}r{\^o}me Louradour, Ronan Collobert, and Jason Weston.
\newblock Curriculum learning.
\newblock In \emph{ICML}, pages 41--48, 2009.

\bibitem[Brockman et~al.(2016)Brockman, Cheung, Pettersson, Schneider,
  Schulman, Tang, and Zaremba]{1606.01540}
Greg Brockman, Vicki Cheung, Ludwig Pettersson, Jonas Schneider, John Schulman,
  Jie Tang, and Wojciech Zaremba.
\newblock Openai gym, 2016.

\bibitem[Camacho et~al.(2019)Camacho, Icarte, Klassen, Valenzano, and
  McIlraith]{camacho2019ltl}
Alberto Camacho, Rodrigo~Toro Icarte, Toryn~Q Klassen, Richard~Anthony
  Valenzano, and Sheila~A McIlraith.
\newblock Ltl and beyond: Formal languages for reward function specification in
  reinforcement learning.
\newblock In \emph{IJCAI}, pages 6065--6073, 2019.

\bibitem[Chen et~al.(2021)Chen, Zhang, Xu, Ma, Yang, Song, Wang, and
  Wu]{chen2021variational}
Jiayu Chen, Yuanxin Zhang, Yuanfan Xu, Huimin Ma, Huazhong Yang, Jiaming Song,
  Yu~Wang, and Yi~Wu.
\newblock Variational automatic curriculum learning for sparse-reward
  cooperative multi-agent problems.
\newblock In \emph{NeurIPS}, 2021.

\bibitem[Dietterich(2000)]{dietterich2000hierarchical}
Thomas~G Dietterich.
\newblock Hierarchical reinforcement learning with the maxq value function
  decomposition.
\newblock \emph{JAIR}, pages 227--303, 2000.

\bibitem[Eimer et~al.(2021)Eimer, Biedenkapp, Hutter, and
  Lindauer]{eimer2021self}
Theresa Eimer, Andr{\'e} Biedenkapp, Frank Hutter, and Marius Lindauer.
\newblock Self-paced context evaluation for contextual reinforcement learning.
\newblock In \emph{ICML}, pages 2948--2958, 2021.

\bibitem[Florensa et~al.(2017)Florensa, Held, Wulfmeier, Zhang, and
  Abbeel]{pmlr-v78-florensa17a}
Carlos Florensa, David Held, Markus Wulfmeier, Michael Zhang, and Pieter
  Abbeel.
\newblock {Reverse Curriculum Generation for Reinforcement Learning}.
\newblock In \emph{{CoRL}}, pages 482--495. {PMLR}, 2017.

\bibitem[Florensa et~al.(2018)Florensa, Held, Geng, and
  Abbeel]{florensa2018automatic}
Carlos Florensa, David Held, Xinyang Geng, and Pieter Abbeel.
\newblock Automatic goal generation for reinforcement learning agents.
\newblock In \emph{ICML}, pages 1515--1528, 2018.

\bibitem[Hallak et~al.(2015)Hallak, Di~Castro, and
  Mannor]{hallak2015contextual}
Assaf Hallak, Dotan Di~Castro, and Shie Mannor.
\newblock Contextual markov decision processes.
\newblock \emph{arXiv preprint arXiv:1502.02259}, 2015.

\bibitem[Icarte et~al.(2018{\natexlab{a}})Icarte, Klassen, Valenzano, and
  McIlraith]{icarte2018using}
Rodrigo~Toro Icarte, Toryn Klassen, Richard Valenzano, and Sheila McIlraith.
\newblock Using reward machines for high-level task specification and
  decomposition in reinforcement learning.
\newblock In \emph{ICML}, pages 2107--2116, 2018{\natexlab{a}}.

\bibitem[Icarte et~al.(2018{\natexlab{b}})Icarte, Klassen, Valenzano, and
  McIlraith]{toro2018teaching}
Rodrigo~Toro Icarte, Toryn~Q Klassen, Richard Valenzano, and Sheila~A
  McIlraith.
\newblock Teaching multiple tasks to an rl agent using ltl.
\newblock In \emph{AAMAS}, pages 452--461, 2018{\natexlab{b}}.

\bibitem[Icarte et~al.(2022)Icarte, Klassen, Valenzano, and
  McIlraith]{icarte2022reward}
Rodrigo~Toro Icarte, Toryn~Q Klassen, Richard Valenzano, and Sheila~A
  McIlraith.
\newblock Reward machines: Exploiting reward function structure in
  reinforcement learning.
\newblock \emph{JAIR}, pages 173--208, 2022.

\bibitem[Jiang et~al.(2015)Jiang, Meng, Zhao, Shan, and
  Hauptmann]{jiang2015self}
Lu~Jiang, Deyu Meng, Qian Zhao, Shiguang Shan, and Alexander~G Hauptmann.
\newblock Self-paced curriculum learning.
\newblock In \emph{AAAI}, 2015.

\bibitem[Jiang et~al.(2021{\natexlab{a}})Jiang, Dennis, Parker-Holder,
  Foerster, Grefenstette, and Rockt{\"a}schel]{jiang2021replay}
Minqi Jiang, Michael Dennis, Jack Parker-Holder, Jakob Foerster, Edward
  Grefenstette, and Tim Rockt{\"a}schel.
\newblock Replay-guided adversarial environment design.
\newblock \emph{NeurIPS}, pages 1884--1897, 2021{\natexlab{a}}.

\bibitem[Jiang et~al.(2021{\natexlab{b}})Jiang, Grefenstette, and
  Rockt{\"a}schel]{jiang2021prioritized}
Minqi Jiang, Edward Grefenstette, and Tim Rockt{\"a}schel.
\newblock Prioritized level replay.
\newblock In \emph{ICML}, pages 4940--4950. PMLR, 2021{\natexlab{b}}.

\bibitem[Klink et~al.(2020{\natexlab{a}})Klink, Abdulsamad, Belousov, and
  Peters]{klink2020self_SPCRL}
Pascal Klink, Hany Abdulsamad, Boris Belousov, and Jan Peters.
\newblock Self-paced contextual reinforcement learning.
\newblock In \emph{CoRL}, pages 513--529, 2020{\natexlab{a}}.

\bibitem[Klink et~al.(2020{\natexlab{b}})Klink, D\textquotesingle~Eramo,
  Peters, and Pajarinen]{NEURIPS2020_68a97503}
Pascal Klink, Carlo D\textquotesingle~Eramo, Jan~R Peters, and Joni Pajarinen.
\newblock Self-paced deep reinforcement learning.
\newblock In \emph{NeurIPS}, pages 9216--9227, 2020{\natexlab{b}}.

\bibitem[Klink et~al.(2021)Klink, Abdulsamad, Belousov, D'Eramo, Peters, and
  Pajarinen]{klink2021probabilistic}
Pascal Klink, Hany Abdulsamad, Boris Belousov, Carlo D'Eramo, Jan Peters, and
  Joni Pajarinen.
\newblock A probabilistic interpretation of self-paced learning with
  applications to reinforcement learning.
\newblock \emph{JMLR}, 22:\penalty0 182:1--182:52, 2021.

\bibitem[Klink et~al.(2022)Klink, Yang, D’Eramo, Peters, and
  Pajarinen]{klink2022curriculum}
Pascal Klink, Haoyi Yang, Carlo D’Eramo, Jan Peters, and Joni Pajarinen.
\newblock Curriculum reinforcement learning via constrained optimal transport.
\newblock In \emph{ICML}, pages 11341--11358, 2022.

\bibitem[Kumar et~al.(2010)Kumar, Packer, and Koller]{NIPS2010_e57c6b95}
M.~Kumar, Benjamin Packer, and Daphne Koller.
\newblock Self-paced learning for latent variable models.
\newblock In \emph{NeurIPS}, 2010.

\bibitem[Kuo et~al.(2020)Kuo, Katz, and Barbu]{kuo2020encoding}
Yen-Ling Kuo, Boris Katz, and Andrei Barbu.
\newblock Encoding formulas as deep networks: Reinforcement learning for
  zero-shot execution of ltl formulas.
\newblock In \emph{IROS}, pages 5604--5610, 2020.

\bibitem[Li et~al.(2017)Li, Vasile, and Belta]{li2017reinforcement}
Xiao Li, Cristian-Ioan Vasile, and Calin Belta.
\newblock Reinforcement learning with temporal logic rewards.
\newblock In \emph{IROS}, pages 3834--3839, 2017.

\bibitem[Littman et~al.(2017)Littman, Topcu, Fu, Isbell, Wen, and
  MacGlashan]{littman2017environment}
Michael~L Littman, Ufuk Topcu, Jie Fu, Charles Isbell, Min Wen, and James
  MacGlashan.
\newblock Environment-independent task specifications via gltl.
\newblock \emph{arXiv preprint arXiv:1704.04341}, 2017.

\bibitem[Narvekar et~al.(2020)Narvekar, Peng, Leonetti, Sinapov, Taylor, and
  Stone]{narvekar2020curriculum}
Sanmit Narvekar, Bei Peng, Matteo Leonetti, Jivko Sinapov, Matthew~E Taylor,
  and Peter Stone.
\newblock Curriculum learning for reinforcement learning domains: A framework
  and survey.
\newblock \emph{JMLR}, pages 1--50, 2020.

\bibitem[Parr and Russell(1997)]{parr1997reinforcement}
Ronald Parr and Stuart Russell.
\newblock Reinforcement learning with hierarchies of machines.
\newblock \emph{NeurIPS}, 1997.

\bibitem[Portelas et~al.(2020)Portelas, Colas, Hofmann, and
  Oudeyer]{portelas2020teacher}
R{\'e}my Portelas, C{\'e}dric Colas, Katja Hofmann, and Pierre-Yves Oudeyer.
\newblock Teacher algorithms for curriculum learning of deep rl in continuously
  parameterized environments.
\newblock In \emph{CoRL}, pages 835--853, 2020.

\bibitem[Racaniere et~al.(2020)Racaniere, Lampinen, Santoro, Reichert, Firoiu,
  and Lillicrap]{racaniere2019automated}
Sebastien Racaniere, Andrew~K Lampinen, Adam Santoro, David~P Reichert, Vlad
  Firoiu, and Timothy~P Lillicrap.
\newblock Automated curricula through setter-solver interactions.
\newblock In \emph{ICLR}, 2020.

\bibitem[Ren et~al.(2018)Ren, Dong, Li, and Chen]{ren2018self}
Zhipeng Ren, Daoyi Dong, Huaxiong Li, and Chunlin Chen.
\newblock Self-paced prioritized curriculum learning with coverage penalty in
  deep reinforcement learning.
\newblock \emph{IEEE TNNLS}, pages 2216--2226, 2018.

\bibitem[Singh(1992)]{singh1992reinforcement}
Satinder~P Singh.
\newblock Reinforcement learning with a hierarchy of abstract models.
\newblock In \emph{National Conference on Artificial Intelligence}, pages
  202--207, 1992.

\bibitem[Sutton et~al.(1999)Sutton, Precup, and Singh]{sutton1999between}
Richard~S Sutton, Doina Precup, and Satinder Singh.
\newblock Between mdps and semi-mdps: A framework for temporal abstraction in
  reinforcement learning.
\newblock \emph{Artificial intelligence}, 112\penalty0 (1-2):\penalty0
  181--211, 1999.

\bibitem[Svetlik et~al.(2017)Svetlik, Leonetti, Sinapov, Shah, Walker, and
  Stone]{svetlik2017automatic}
Maxwell Svetlik, Matteo Leonetti, Jivko Sinapov, Rishi Shah, Nick Walker, and
  Peter Stone.
\newblock Automatic curriculum graph generation for reinforcement learning
  agents.
\newblock In \emph{AAAI}, 2017.

\bibitem[Velasquez et~al.(2021)Velasquez, Bissey, Barak, Beckus, Alkhouri,
  Melcer, and Atia]{velasquez2021dynamic}
Alvaro Velasquez, Brett Bissey, Lior Barak, Andre Beckus, Ismail Alkhouri,
  Daniel Melcer, and George Atia.
\newblock Dynamic automaton-guided reward shaping for monte carlo tree search.
\newblock In \emph{AAAI}, 2021.

\bibitem[Xu and Topcu(2019)]{xu2019transfer}
Zhe Xu and Ufuk Topcu.
\newblock Transfer of temporal logic formulas in reinforcement learning.
\newblock In \emph{IJCAI}, page 4010, 2019.

\bibitem[Xu et~al.(2020)Xu, Gavran, Ahmad, Majumdar, Neider, Topcu, and
  Wu]{xu2020joint}
Zhe Xu, Ivan Gavran, Yousef Ahmad, Rupak Majumdar, Daniel Neider, Ufuk Topcu,
  and Bo~Wu.
\newblock Joint inference of reward machines and policies for reinforcement
  learning.
\newblock In \emph{ICAPS}, pages 590--598, 2020.

\bibitem[Zhang et~al.(2020)Zhang, Abbeel, and Pinto]{NEURIPS2020_566f0ea4}
Yunzhi Zhang, Pieter Abbeel, and Lerrel Pinto.
\newblock Automatic curriculum learning through value disagreement.
\newblock In \emph{NeurIPS}, pages 7648--7659, 2020.

\bibitem[Zheng et~al.(2022)Zheng, Yu, and Zhang]{zheng2022lifelong}
Xuejing Zheng, Chao Yu, and Minjie Zhang.
\newblock Lifelong reinforcement learning with temporal logic formulas and
  reward machines.
\newblock \emph{Knowledge-Based Systems}, page 109650, 2022.

\end{thebibliography}

\end{document}